\documentclass[twoside]{article}

\usepackage[accepted]{aistats2026}

\usepackage{
    amsmath,
    amssymb,
    amsthm,
    booktabs,
    enumitem,
    graphicx,
    natbib,
    xcolor,
    xurl,
    mathtools,
    bm,
}

\usepackage[hyperfootnotes=false]{hyperref}

\usepackage[
    capitalize,
    nameinlink,
    noabbrev,
]{cleveref}

\AddToHook{cmd/appendix/before}{%
    \crefalias{section}{appendix}%
    \crefalias{subsection}{subappendix}%
    \crefalias{subsubsection}{subsubappendix}%
}

\usepackage[
    font=small,
    labelfont=bf,
    labelsep=quad,
    justification=justified,
    singlelinecheck=false,
]{caption}

\usepackage[
    font=small,
    labelfont=bf,
    labelsep=quad,
]{subcaption}

\creflabelformat{equation}{#2\textup{#1}#3}

\hypersetup{
    colorlinks,
    linkcolor={blue!50!black},
    citecolor={blue!50!black},
    urlcolor={blue!50!black},
}

\newcommand{\shorturl}[1]{\href{https://#1}{\nolinkurl{#1}}}

\newcommand{\SEM}[1]{\textnormal{\scriptsize $\pm$\,#1}}
\newcommand{\red}[1]{\textcolor[HTML]{D1495B}{#1}}
\newcommand{\blue}[1]{\textcolor[HTML]{30638E}{#1}}

\newtheorem{proposition}{Proposition}
\newtheorem{corollary}{Corollary}
\newtheorem{theorem}{Theorem}
\newtheorem*{theorem*}{Theorem}
\newtheorem*{corollary*}{Corollary}

\newtheorem*{remark}{Remark}

\theoremstyle{definition}
\newtheorem{definition}{Definition}

\newcommand{\argmin}{\mathop{\mathrm{arg\,min}}}

\newcommand{\entropy}[1]{\mathrm{H} \! \left[ #1 \right]}
\newcommand{\expectation}[2]{\mathbb{E}_{#1} \! \left[ #2 \right]}

\newcommand{\uncertainty}[3]{h_{#1}^{#2} \! \left[ #3 \right]}

\newcommand{\uncertaintyphiw}[1]{h_\phi^w \! \left[ #1 \right]}
\newcommand{\variance}[2]{\mathbb{V}_{#1} \! \left[ #2 \right]}

\newcommand{\E}{\mathbb{E}}

\newcommand{\bbR}[0]{\mathbb{R}}

\newcommand{\calA}[0]{\mathcal{A}}

\newcommand{\calX}[0]{\mathcal{X}}
\newcommand{\calY}[0]{\mathcal{Y}}
\newcommand{\calZ}[0]{\mathcal{Z}}

\newcommand{\EIG}[0]{\mathrm{EIG}}
\newcommand{\EPU}[0]{\mathrm{EPU}}
\newcommand{\EUR}[0]{\mathrm{EUR}}

\begin{document}
    \raggedbottom

    \twocolumn[
        \aistatstitle{Loss-Driven Bayesian Active Learning}
        \aistatsauthor{Zhuoyue Huang \And Freddie Bickford Smith \And Tom Rainforth}
        \aistatsaddress{University of Oxford \And University of Oxford \And University of Oxford}
    ]

    \begin{abstract}
        The central goal of active learning is to gather data that maximises downstream predictive performance, but popular approaches have limited flexibility in customising this data acquisition to different downstream problems and losses.
We propose a rigorous loss-driven approach to Bayesian active learning that allows data acquisition to directly target the loss associated with a given decision problem.
In particular, we show how any loss can be used to derive a unique objective for optimal data acquisition.
Critically, we then show that any loss taking the form of a weighted Bregman divergence permits analytic computation of a central component of its corresponding objective, making the approach applicable in practice.
In regression and classification experiments with a range of different losses, we find our approach reduces test losses relative to existing techniques.
    \end{abstract}

    \section{INTRODUCTION}\label{sec:introduction}

In machine learning we commonly use a loss function (or equivalently a reward or utility function) to quantify a model's performance \citep{murphy2022probabilistic}.
Formally a loss function encodes preferences over the outcomes of a given decision \citep{vonneumann1947theory}.
It could for instance measure the accuracy of a prediction, the impact of a medical treatment on a patient's health, or the return on a financial investment.

The diversity of losses in use across the field reflects the variety of decision problems we want to solve.
For example, predictive decision problems vary in the significance assigned to different prediction errors.
In protein design \citep{notin2024machine} we might care more about accurately predicting a desirable property (eg, binding affinity) when it has a high value, and in safety-critical settings \citep{roy1952safety} we might want to penalise optimistic predictions more than pessimistic ones.

\looseness=-1
We argue that rigorously tailoring active learning to this breadth of possible decision problems is not possible with the approaches most popularly used in the field.
In particular, while a range of data-acquisition objectives have been proposed \citep{li2024survey,settles2012active}, in general these either have limited flexibility in targeting different downstream losses or rely on nested training of the downstream model within the acquisition function, meaning they can only be used with particular models or substantial approximations.
For example, information-theoretic objectives \citep{bickfordsmith2023prediction,mackay1992information} and Bayesian variants of popular variance-based objectives \citep{cohn1993neural,cohn1994active} are fixed in the losses they target (log loss and quadratic loss respectively), while popular error-based objectives \citep{roy2001toward} allow for different losses but are heavily restricted in the models they can be practically used with.
There is thus a pressing need for flexible, principled methods that directly optimise downstream loss in a range of decision problems without model restrictions.

\looseness=-1
We address this by proposing an explicitly loss-driven approach to active learning based on Bayesian decision theory \citep{berger1985statistical,ramsey1926truth,savage1954foundations}.
Specifically, we revisit the decision-theoretic foundations of Bayesian experimental design \citep{degroot1962uncertainty,lindley1972bayesian,raiffa1961applied} to show how data utility can be formalised as the Bayesian expected loss of a Bayes-optimal downstream action informed by the acquired data.
This then reveals that any model-loss pairing defines a unique objective whose minimiser achieves theoretically optimal data acquisition.

Objectives derived from our framework are not practically usable in the fully general case: they comprise an expectation of a minimisation of an inner expectation, which is rarely computationally viable to directly optimise.
However, we show that any loss that can be written as a weighted Bregman divergence \citep{bregman1967relaxation} allows us to compute the inner minimisation analytically, such that the overall objective can be simplified to a form amenable to practical estimation and optimisation.
Given that this broad class of losses includes many losses used in practice, this then yields a powerful and highly applicable class of loss-driven Bayesian active learning methods.

To help demonstrate the performance benefits of our approach, we introduce two concrete data-acquisition objectives with prediction-space weighting: one analogous to information-based objectives \citep{bickfordsmith2023prediction,bickfordsmith2024making,mackay1992evidence,mackay1992information} and one analogous to Bayesian variants of variance-based objectives \citep{cohn1993neural,cohn1994active}.
On classification and regression problems these loss-matched objectives improve the corresponding (weighted) test losses, and a same-task comparison across downstream losses (quadratic vs Linex) highlights the cost of mismatch.

    \section{BACKGROUND}\label{sec:background}

\textbf{Bayesian decision theory} \citep{berger1985statistical,ramsey1926truth,savage1954foundations} provides a rigorous framework for choosing an action, $a \in \calA$, under imperfect knowledge of a world state, $z \in \calZ$.
A decision-maker's loss function, $\ell:  \calZ \times \calA \to \bbR$, encodes their preferences over outcomes, and their model, $p(z)$, encodes their beliefs about the world state.
Any given action can be judged in terms of its Bayesian expected loss,
\begin{align}
    \label{eq:general_expected_loss}
    L(a)
    =
    \expectation{p(z)}{\ell(z,a)}
    .
\end{align}
A Bayes-optimal action minimises this expected loss:
\begin{align}
    \label{eq:bayes_optimal_action}
    a^*
    =
    \argmin_{a \in \calA} L(a)
    .
\end{align}

\looseness=-1
\textbf{Bayesian experimental design (BED)} \citep{degroot1962uncertainty,degroot1970optimal,lindley1956measure,lindley1972bayesian,raiffa1961applied,raiffa1968decision} applies Bayesian decision theory to the problem of designing experiments.
The design, $\xi$, of an abstractly defined experiment represents controllable aspects of the data generation, and our choice of $\xi$ affects the data, $d$, we will observe.
In the original formulation of BED~\citep{lindley1956measure}, designs are judged in terms of the expected information gain (EIG) in some quantity of interest, $\psi$, from observing $d$:
\begin{align}
    \label{eq:eig}
    \EIG_\psi(\xi)
    =
    \expectation{p(d; \xi)}{\entropy{p(\psi)} - \entropy{p(\psi | d; \xi)}},
\end{align}
where $\entropy{\cdot}$ denotes Shannon entropy \citep{shannon1948mathematical}.
Here and elsewhere we use $p(a; b)$ to denote a distribution over $a$ that depends on $b$ without necessarily being the conditional distribution, $p(a | b)$, associated with an explicit joint distribution, $p(a, b)$.

BED more generally involves choosing a data policy, $\pi_d$, that minimises a Bayesian expected loss of the form
\begin{align*}
    L_\mathrm{BED}(\pi_d)
    =
    \expectation{p(d, \psi; \pi_d)}{\ell_d(\pi_d, d, \psi)}
    .
\end{align*}
Here $\pi_d$ could represent a fixed sequence of designs or a more complex decision-making procedure that adaptively selects designs based on the experiment history \citep{foster2021deep,huan2016sequential}.
Meanwhile $\ell_d$ is defined not in terms of downstream actions but in terms of the data policy and the data it yields, along with the quantity of interest.
Determining what $\ell_d$ should be can be challenging; the predominant approach is to use $\ell_d(\pi_d, d, \psi) = \log p(\psi) - \log p(\psi | d; \pi_d)$, recovering the EIG \citep{rainforth2024modern}.

\looseness=-1
\textbf{Bayesian active learning} \citep{bickfordsmith2023prediction,bickfordsmith2024making,gal2017deep,houlsby2011bayesian,kirsch2019batchbald,mackay1992evidence,mackay1992information} is the application of BED to sequentially selecting data to use in training a predictive model for outputs, $\tilde{y} \in \calY$ given inputs $\tilde{x} \in \calX$.
Typically each design, $\xi$, corresponds to choosing one or more inputs, $\xi \in \mathcal{X}^n$, and the data corresponds to the corresponding observed output label(s).
This process is then iterated, updating the model with the new data before again choosing new inputs to label.
Typically the inputs for labelling are selected from some large \emph{pool} of unlabelled examples.

The most popular approaches use myopic information-theoretic data acquisition, each step of which involves choosing $\xi$ to maximise $\EIG_\psi(\xi)$, namely the expected uncertainty reduction in $\psi$ from performing a Bayesian update on $\psi$ given the observed $d$, with uncertainty measured by Shannon entropy (\Cref{eq:eig}).
If $\psi = \theta$ represents a set of stochastic model parameters then the EIG is known as the BALD score \citep{houlsby2011bayesian}; if instead $\psi = (\tilde{x}, \tilde{y})$ represents a downstream input-output pair then we get the expected predictive information gain \citep[EPIG;][]{bickfordsmith2023prediction}.

The \textbf{variance-based objective} introduced by \citet{cohn1993neural} similarly corresponds to an expected reduction in uncertainty, but with uncertainty measured by variance rather than Shannon entropy and with generic belief updating on new data rather than a Bayesian update.
With $p(\tilde{x})$ denoting beliefs over downstream inputs, the expected variance reduction (EVR) is
\begin{align*}
    \mathrm{EVR}(\xi)
    =
    \expectation{p(d; \xi)p(\tilde{x})}{\variance{p(\tilde{y}; \tilde{x})}{\tilde{y}} - \variance{p(\tilde{y}; \tilde{x}, d)}{\tilde{y}}}
    .
\end{align*}

The \textbf{expected future error} (EFE) framework introduced by \citet{roy2001toward} is based around measuring the future expected classification loss of predictions on the unlabelled pool data, $\mathcal{X}_p$.
Specifically, they consider retraining the model on a hypothetical new datapoint, $(x,y)$, for each $x\in\mathcal{X}_p$ and all possible class labels, $y\in\mathcal{Y}$, evaluating some loss function for the updated model's predictions on all other points in the pool, and then choosing the input that gives the lowest average loss in expectation over possible labels.
That is, they choose the $x\in\mathcal{X}_p$ that minimises
\begin{align*}
    \mathrm{EFE}(x) = \expectation{p(y;x)}{
        \sum_{\tilde{x}\in \mathcal{X}_p} \frac{\expectation{p(\tilde{y};\tilde{x},x,y)}{\ell (p(\tilde{y};\tilde{x},x,y),\tilde{y})}}{{|\calX_p|}}
    },
\end{align*}
\looseness=-1
where $p(y;x)$ is based on the current model, each $p(\tilde{y};\tilde{x},x,y)$ is a retrained model with the new hypothetical data, and $\ell$ is the assumed classification loss.

In principle, their framework is applicable to any downstream loss defined on the same space as their training data.
However, it requires a nested retraining of the model for each possible $(x,y)$, meaning it is not generally applicable in practice.
Indeed, their specific implementations, which focus on log loss and zero-one loss, rely on using models with cheap incremental updates, along with additional approximations.

A \textbf{Bregman divergence} \citep{bregman1967relaxation} between elements in a convex set, $x, y \in \Omega \subseteq \bbR^K$, is defined as
\begin{align*}
    D_\phi(x, y)
    =
    \phi(x) - \phi(y) - \langle \nabla \phi(y), x - y \rangle,
\end{align*}
where $\phi: \Omega \to \bbR$ is a differentiable and strictly convex potential function, and $\langle \cdot , \cdot \rangle$ denotes an inner product.
It is nonnegative, typically asymmetric in $x$ and $y$ (and thus is not a metric), and vanishes if and only if $x = y$.

Two important examples of Bregman divergences are Mahalanobis distance and KL divergence \citep{kullback1951information}.
Specifically, setting $\Omega = \bbR^K$ and $\phi(x) = x^\top A x$ for $A \succ 0$, yields $D_\phi(x, y) = (x - y)^\top A (x - y)$~\citep{banerjee2005clustering}, which reduces to the squared error when $A = I_K$ is the $K$-dimensional identity matrix.
Meanwhile, setting $\Omega = \Delta^{K-1}$ and $\phi(x) = \sum_{j=1}^K x_j \log x_j$ gives the KL divergence,  $D_\phi(x, y) = \sum_{j=1}^K x_j (\log x_j - \log y_j)$, which reduces to the negative log likelihood (NLL), $-\log y_i$, when $x = e_i$ is the $i$th standard basis vector.

Notably the results we present apply not just to Bregman divergences between finite-dimensional vectors, as above, but also to Bregman divergences between functions \citep{frigyik2008functional}.
This is practically relevant in many machine-learning settings.
For example, if we are working with probability densities then we require a functional Bregman divergence to extend the KL divergence from above to this setting \citep{csiszar2012generalized}.
We provide details in \Cref{sec:functional_bregman}.

In probabilistic prediction, the regret of a differentiable strictly proper scoring rule is exactly a Bregman divergence of the associated generalised entropy, so Bregman geometry is the canonical regret geometry of proper probabilistic evaluation. See \cref{app:proper} for further discussion on the links between proper scoring rules and Bregman divergences.

    \section{LOSS-DRIVEN ACQUISITION}\label{sec:method}

The approaches to Bayesian experimental design and active learning presented in \Cref{sec:background} have a fundamental shortfall: they are difficult to customise to different decision problems that we might want to solve using the data we are acquiring.
In particular, it is not immediately clear how to align the loss on acquired data with our ultimate goal of minimising loss on ``terminal'' (downstream) actions.

To address this, we now revisit BED from the perspective of decision-making and prediction.
Specifically, by revisiting the Bayesian decision theory that underlies BED, we identify a loss-driven approach to data acquisition that in the special case of a negative-log-likelihood loss is equivalent to EIG maximisation.

\subsection{Bayes-optimal acquisition}\label{sec:bayes_data_acquisition}

We consider the end-to-end problem of acquiring data and then using that data to take a terminal action.
To this end, we first consider the optimal downstream action if data $d$ is acquired, then work backwards to figure out the best way to gather data from this.

Assume we have some terminal loss function, $\ell: \calZ \times \calA \to \bbR$, that depends on our terminal action, $a \in \mathcal{A}$, and true world state, $z\in\mathcal{Z}$ (in principle the loss can also directly depend on $(d,\pi_d)$, but we will omit such dependency for simplicity).
Further, let our model's belief over $z$ after it has observed some hypothetical data, $d$, be given by $p(z|d;\pi_d) \propto p(z)p(d|z;\pi_d)$, where $p(z)$ encodes our model’s current beliefs and $p(d|z;\pi_d)$ is our predictive model for new data.
The downstream Bayes-optimal action is now given by
\begin{align}
\label{eq:bayes_act}
    a^* = \argmin_{a \in \calA} \expectation{p(z | d; \pi_d)}{\ell(z, a)}.
\end{align}
The Bayesian expected loss associated with this Bayes-optimal action corresponds to a \emph{generalised entropy}~\citep{bickfordsmith2025rethinking,dawid1998coherent}:
\begin{align}
\label{eq:gen_ent}
    h_{\ell}[p(z|d;\pi_d)] = \min_{a \in \calA} \expectation{p(z | d; \pi_d)}{\ell(z, a)}.
\end{align}
This provides a \emph{loss-driven} measure of uncertainty on our posterior beliefs $p(z|d;\pi_d)$.
It represents the best expected performance we can achieve given the data we have managed to acquire and the data policy we used.
Better data provides more information about the state of the world, $z$, reducing our uncertainty and, in turn, allowing for better downstream actions.

Moreover, sequential application of Bayesian decision theory dictates that we should assume downstream actions are made Bayes-optimally when making earlier decisions to preserve coherence~\citep{lindley1972bayesian}.
Thus, from a Bayesian decision theory perspective, the generalised entropy in~\Cref{eq:gen_ent} is \emph{the} canonical measure we should use to assess how effective our data acquisition policy has been once the data has been observed.

From here it is easy to see that the Bayes-optimal data gathering policy for a given model-loss pairing and allowed space of policies $\Pi_d$
is
\begin{align}
    &\pi_d^* = \argmin_{\pi_d \in \Pi_d}
    \EPU_{p(d, z; \pi_d)}^{\ell} \nonumber \\
    \mathrm{where} \quad &\EPU_{p(d, z; \pi_d)}^{\ell} = \expectation{p(d;\pi_d)}{h_{\ell}[p(z|d;\pi_d)]} 
    \label{eq:data_objective_explicit_min}
\end{align}
with $p(d;\pi_d) = \expectation{p(z)}{p(d|z;\pi_d)}$. 
Here $\EPU_{p(d, z; \pi_d)}^{\ell}$ denotes the \emph{expected posterior uncertainty} (EPU) in $z$ for a given loss function, $\ell$, and a given joint model over $d$ and $z$, $p(d, z; \pi_d)$.
If our loss depends only on $(z, a)$ and not directly on $(\pi_d,d)$, then this is equivalent to the expected uncertainty reduction (EUR),
$$\EUR_{p(d, z; \pi_d)}^{\ell} = \uncertainty{\ell}{}{p(z)} - \EPU_{p(d, z; \pi_d)}^{\ell}.$$
Equivalent notions to the EPU and EUR have been considered in past work \citep{bickfordsmith2025rethinking,dawid1998coherent,huang2024amortized,neiswanger2022generalizing}; the most recent of these relaxed the form of the future beliefs to be potentially non-Bayesian.
While principled, the definition of the generalised entropy involves a minimisation, creating a nested problem where we must minimise the expectation of a minimisation of an inner expectation that is itself with respect to an intractable posterior.
Using it as a practical active learning objective thus further requires a mechanism to get around this nested minimisation.

\begin{table*}[t]
    \centering
    \small
    \begin{tabular}{llllll}
    \toprule
    Name & Example & $\calZ$ & $T(z)$ & $\phi(z)$ & $D_\phi(T(z), T(b))$ \\
    \midrule
    Squared error & \citet{berger1985statistical} & $\bbR$ & $z$ & $z^2$ & $(z-b)^2$ \\
    Box-Cox squared error & \citet{box1964analysis} & $\bbR^+$ & $(z^\lambda - 1) / \lambda, \lambda \neq 0$ & $z^2$ & $(z^\lambda - b^\lambda)^2 / \lambda^2$ \\
    Linex & \citet{zellner1986bayesian} & $\bbR^+$ & $\exp(-\alpha z), \alpha > 0$ & $-\log z$ & $\exp(\alpha(b - z)) - \alpha(b - z) - 1$ \\
    \bottomrule
    \end{tabular}
    \caption{
        Standard predictive losses can be expressed as a Bregman divergence in a transformed space, $D_\phi(T(z), T(b))$, where $z \in \calZ$ is a world state, $b\in\calZ$ is an in-space action, $\phi$ is a convex function and $T: \calZ \to \mathrm{dom}(\phi)$ is measurable.
        \vspace{-10pt}
    }
    \label{tab:bregman_losses}
\end{table*}

\subsection{Using Bregman-divergence losses}\label{sec:bregman_losses}

Our key insight is now to show that the EPU defined in~\Cref{eq:data_objective_explicit_min} is in fact a viable objective for practical active learning approaches for a wide range of losses.
Specifically, if our loss takes the form
\begin{align}
    \label{eq:weighted_bregman_loss}
    \ell(z, a)
    =
    w(z) D_\phi(T(z), a),
\end{align}
where $w: \calZ \to \bbR_+$ is a weighting function, $D_\phi$ is the Bregman divergence associated with some strictly convex differentiable potential function, $\phi: \Omega \to \bbR$, and $T: \calZ \to \mathrm{dom}(\phi)$ is measurable, then the required inner minimisation can be performed analytically, leading to the following core result.\footnote{Our results generally extend to cases where the loss depends directly on $(\pi_d,d)$, provided each possible $(\pi_d,d)$ realisation produces a loss that takes the form in~\Cref{eq:weighted_bregman_loss} (potentially with different $w$, $\phi$, and/or $T$).}

\vspace{5pt}

\begin{theorem}
\label{the:main_theorem}
Assume the terminal loss takes the form of a weighted Bregman divergence on measurable transformations of the world state as per~\Cref{eq:weighted_bregman_loss}, and that $\expectation{q(z)}{w(z)}$, $\expectation{q(z)}{w(z)T(z)}$ and $\expectation{q(z)}{w(z)\phi(T(z))}$ are finite for some generic distribution $q(z)$. If $\mathcal{A} \subseteq \mathrm{ri}(\mathrm{dom}(\phi))$ is convex and contains $\expectation{q_w(z)}{T(z)}$ where  $q_w(z) \!=\! w(z) q(z) / \bar{w}_q$ and $\bar{w}_q \!=\! \expectation{q(z)}{w(z)}$, then the generalised entropy associated with this $q$ is given by
\begin{align}
    & \uncertainty{\phi,T}{w}{q(z)} := \min_{a \in \calA} \expectation{q(z)}{w(z) D_\phi(T(z), a)}
    \nonumber
    \\
    & =
    \bar{w}_q \left( \expectation{q_w(z)}{\phi(T(z))} - \phi(\expectation{q_w(z)}{T(z)}) \right).
    \label{eq:analytic_uncertainty}
\end{align}
Assuming that $p(z|d;\pi_d)$ satisfies the above restrictions of $q$ for all $(d,\pi_d)$,
then the Bayes-optimal data gathering policy for our model-loss pairing is given by
\begin{align}
\label{eq:opt_policy_breg}
\begin{split}
    \pi_d^\ast = \argmin_{\pi_d\in\Pi_d} ~\bar{w}\,\mathbb{E}_{p_w(d;\pi_d)}&\Big[
        \expectation{p_w(z | d;\pi_d)}{\phi(T(z))} \\ -& \phi(\expectation{p_w(z | d;\pi_d)}{T(z)}) \Big],
\end{split}
\end{align}
where we define beliefs $p_w(d;\pi_d)=\bar w(d,\pi_d)p(d ; \pi_d) /\bar w$, $p_w(z | d ; \pi_d)  =  w(z) p(z | d ; \pi_d) / \bar{w}(d, \pi_d)$, and weights $\bar{w}(d,\pi_d)  =  \expectation{p(z | d ; \pi_d)}{w(z)}$, $\bar w = \expectation{p(z)}{w(z)}$.
\end{theorem}

This result (see \Cref{app:proof_uncertainty_epu_general_form} for proof) means \Cref{eq:opt_policy_breg} is now something that we can realistically target for data acquisition, as it no longer requires us to perform a nested optimisation within the acquisition function (though it is still in general a nested expectation~\citep{rainforth2018nesting}).
The implications of this are considerable because \Cref{eq:weighted_bregman_loss} generalises a number of standard predictive losses. In addition to the KL divergence (reducing to NLL) and Mahalanobis distance (reducing to squared error) mentioned in \Cref{sec:background}, some further examples are presented in \Cref{tab:bregman_losses}.

\vspace{5pt}
\looseness=-1
\begin{remark}
The Bayes-optimal action when using a loss of the form~\Cref{eq:weighted_bregman_loss} with constant weight function $w$ is given by the posterior mean of $T(z)$.
Moreover, this result goes both ways: under mild regularity conditions, if our Bayes act is a posterior mean, this itself implies that our loss function must be a Bregman divergence loss of the form of~\Cref{eq:weighted_bregman_loss} up to additive constant terms~\citep{banerjee2005optimality}.
\end{remark}

\looseness=-1
The above remark explains why Bregman-type losses are not only numerically convenient, but also a natural fit for constructing a loss-driven active learning framework: they align with mean-oriented decision semantics and provide the analytic structure needed for tractable acquisition.
The inclusion of the weighting factor, $w(z)$, increases the flexibility of the framework, allowing us to target scenarios where the optimal act is not to use the mean under our beliefs because of asymmetry in losses incurred by different errors (e.g.~penalising under-prediction more than over-prediction).

The transformed-space formulation also gives a unified treatment of finite categorical targets: taking $T(z)=e_z$ for a true class label, $z$, maps labels to the simplex, so the Bayesian act is the class-probability vector, $\expectation{p(z)}{e_z}=p\in\bbR^K$, and the one-hot case follows from the same computation rather than requiring a separate construction.
For example, using $\phi(x) = \sum_{i=1}^K x_i\log x_i$ and $T(z)=e_z$ allows us to recover Shannon entropy as our $h_\ell$ and thus an EUR that corresponds to the EIG in $z$.

More broadly, the transformation enlarges the class of tractable Bregman-type losses because, for data acquisition, we only need the corresponding uncertainty, $\uncertainty{\phi}{}{\cdot}$, to evaluate the acquisition function, not necessarily a minimiser that lies in the original world space $\calZ$. When such an in-space minimiser is required, i.e., $z,b\in\calZ$, one can instead use the pullback loss $D_\phi(T(z),T(b))$, and, assuming $T$ is injective and $T(\calZ)$ is convex, recover the explicit Bayesian act $b^* = T^{-1}(\expectation{p(z)}{T(z)})$ (see \cref{tab:bregman_losses} for examples).

\subsection{Focusing on active learning}\label{sec:active_learning}

Now we consider an active-learning setting in which $z$ represents an output variable of interest, $\pi_d = x$ represents an input for labelling and $d = y$ represents the label of that input.
We introduce the notion of a context variable, $c$, that relates to the output we want to predict; it could for example represent a test input.

The setup we have considered so far implicitly allows for transductive learning \citep{vapnik1982estimation}, in which $c$ is known upfront: $c$ can be used to inform our joint beliefs over $y$ and $z$, and it does not need to be explicitly denoted.
Our focus is instead going to be on settings where $c$ is unknown at the time of data acquisition but known at the time of choosing a terminal action.

\looseness=-1
Extending to this new setting simply requires us to define our joint model to incorporate the unknown $c$, leading to a new form of EPU.
We define our joint model to be $p(c, y, z | x) = p(c) p(z | c) p(y | c, x, z)$ and, assuming $p(y | x, c) = p(y | x)$ for simplicity, this gives us 
\begin{align*}
    \EPU_{p(c, y, z | x)}^{\phi,T,w}
    =
    \expectation{p(c) p(y | x)}{\uncertainty{\phi,T}{w}{p(z | c, x, y)}}
    ,
\end{align*}
which is a function of the $x$ that we want to choose.
As noted in \Cref{sec:bayes_data_acquisition}, we can equivalently minimise this EPU or maximise a corresponding EUR, given by 
\begin{align}
    \EUR_{p(c, y, z | x)}^{\phi,T,w}
     =
    \expectation{p(c)}{\uncertainty{\phi,T}{w}{p(z | c)}} - \EPU_{p(c, y, z | x)}^{\phi,T,w}
    \nonumber
    ,
\end{align}
which is the negative EPU plus a term constant in $x$.

It is worth noting a caveat around using these forms of EPU and EUR defined in terms of a single step of data acquisition.
This follows the standard practice in Bayesian active learning to repeatedly use a single-step acquisition objective to choose a sequence of inputs for labelling (\Cref{sec:background}).
However, it means it does not conform to true Bayesian optimality as per~\Cref{eq:opt_policy_breg}, which would require us to go through the expensive process of fully planning ahead through all future data acquisitions as would be required, e.g.~using deep adaptive design~\citep{foster2021deep}.
Our results show it can nevertheless perform well and provide a practical loss-driven Bayesian active learning approach.

\subsection{Two concrete objectives}\label{sec:concrete_objectives}

We next use our setup to identify two new example objectives from our framework that are related to existing information-based and variance-based objectives.
We obtain the first objective using a weighted NLL loss, produced by $T(z)=e_z$ and $\phi(x)=\sum_{i=1}^K x_i \log x_i$:
\begin{align*}
    \ell_\mathrm{WNLL}(z,a)
    =
    - w(z) \log a_z
\end{align*}
where $a_z$ denotes the probability mass assigned to $z$ by the distribution that $a$ represents.
The resulting uncertainty for a generic $q(z | c)$ is $\uncertainty{\mathrm{WNLL}}{}{q(z | c)} = \bar{w}(c) \entropy{q_w(z | c)}$ where we define $\bar{w}(c)$ and $q_w(z | c)$ as in \Cref{eq:analytic_uncertainty} except with $q(z)$ replaced by $q(z | c)$.
The corresponding EUR is a ``weighted EPIG'',
\begin{align*}
  &\mathrm{EPIG}_w(x) =\\  &\quad\expectation{p(c) p(y | x)}{\bar{w}(c)\entropy{p_w(z | c)} - \bar{w}(c,x,y)\entropy{p_w(z | c, x, y)}}
\end{align*}
which is equal to EPIG if $w(z) = 1$ for all $z$.

Next we use a weighted squared error:
\begin{align*}
    \ell_\mathrm{WSE}(z,a)
    =
    w(z) (z-a)^2
    .
\end{align*}
Similar to before, our uncertainty is $\uncertainty{\mathrm{WSE}}{}{q(z | c)} = \bar{w}(c) \variance{q_w(z | c)}{z}$, which leads to a ``weighted EVR'',
\begin{align*}
    &\mathrm{EVR}_w(x) = \\
    &\quad\expectation{p(c) p(y | x)}{\bar{w}(c) \variance{p_w(z | c)}{z} - \bar{w}(c,x,y)\variance{p_w(z | c, x, y)}{z}}
    .
\end{align*}
If, as above, we consider $w(z) = 1$ for all $z$ then this objective is a Bayesian variant of the expected variance reduction proposed by \citet{cohn1993neural}.

\looseness=-1
\paragraph{Choosing $w$ in practice} Our framework is agnostic to how the weight function $w$ is obtained: $w$ is simply a user-specified encoding of which values of $z$ matter more for downstream accuracy. In practice, $w$ may be hand-specified from domain knowledge, derived from an explicit cost model, or elicited from user preferences. When the downstream evaluation loss is already fixed, $w$ should be chosen to match that loss.

\subsection{Insights on prediction-space weighting}

There are important links between the EPU produced by our weighted-Bregman losses and the EPU of corresponding unweighted losses under an alternative model, as explained in the following result.

\vspace{5pt}
\begin{corollary}
    \label{prop:epu_weight}
    Assume the generalised entropy terms are well defined. Then, with $\bar w(c)=\expectation{p(z|c)}{w(z)}$ and $p_w(z|c)=w(z)\,p(z|c) / \bar w(c)$,
    \begin{align*}
    \mathrm{EPU}^{\phi,T,w}_{p(z,y|c,x)}=\bar w(c)\,
    \mathrm{EPU}^{\phi,T}_{p_w(z| c)\,p(y| x,z,c)},
    \end{align*}
    where the lack of the $w$ superscript in the second $\mathrm{EPU}$ is used to imply $w(z)=1, \forall z$.
    Moreover, averaging over contexts gives
    \begin{align}
    \mathrm{EPU}^{\phi,T,w}_{p(c,y,z|x)}&=\bar w\,\mathrm{EPU}^{\phi, T}_{p_w(c)\,p_w(z| c)\,p(y| x,z,c)}, \label{eq:reweighted_prior_1} \\
    &= \bar w\,\mathrm{EPU}^{\phi, T}_{p_w(z)\,p(c,y| x,z)}, \label{eq:reweighted_prior_2}
    \end{align}
    where $\bar w=\mathbb{E}_{p(c)}[\bar w(c)]$ and $p_w(c)=\bar w(c)\,p(c)/\bar w$. The same identities hold with $\mathrm{EPU}$ replaced by $\mathrm{EUR}$.
\end{corollary}

This result (see \cref{app:weighted_EPU_proof} for proof) 
makes precise what prediction-space weighting does to the active-learning objective:
weighting the downstream loss by $w(z)$ is equivalent to evaluating the \emph{unweighted} EPU under an alternative model where 
we have replaced the \emph{marginal} prior on $z$, $p(z)$, with the weighted prior, $p_w(z)\propto p(z)w(z)$, while keeping the conditional distribution on all other variables, $p(c,y|x,z)$, fixed.
Thus, if our original model was based on directly placing a prior on $z$ and then a likelihood on data and contexts given $z$ (e.g.~for BALD, $z$ corresponds to the model parameters and $c=\emptyset$), we can easily just adjust our prior instead of imposing any kind of weighting.

However, in many cases our prior on $z$ is not specified directly but is the pushforward of some other distribution.
For example, when using EPIG then our model is constructed by first defining $p(c)$, a prior on underlying parameters $p(\theta)$, and a shared predictive distribution for outputs given inputs parameterised by $\theta$, $p_{\mathrm{pred}}(\mathrm{out}|\theta,\mathrm{in})$.  The required model terms in~\Cref{prop:epu_weight} are then derived indirectly from these, namely $p(z|c) = \expectation{p(\theta)}{p_{\mathrm{pred}}(\mathrm{out}=z|\theta,\mathrm{in}=c)}$, $p(z)=\expectation{p(c)}{p(z|c)}$, $p(\theta|z,c) =p(\theta)p_{\mathrm{pred}}(\mathrm{out}=z|\theta,\mathrm{in}=c)/p(z|c)$, and  $p(y|x,z,c)=\expectation{p(\theta|z,c)}{p_{\mathrm{pred}}(\mathrm{out}=y|\theta,\mathrm{in}=x)}$.
Here it is not clear in general (at least without solving an inverse problem) how to adjust $p(\theta)$ and $p(c)$ to induce a desired change from $p(z)$ to $p_w(z)$, particularly given we need to also leave $p(y|x,z,c)$ unchanged.
The weighted EPU form is thus essential in allowing us to formulate our model--loss pairing in practice.

\Cref{eq:reweighted_prior_1} further shows that when we average over contexts, then the weighting on $z$ induces a change in the \emph{effective} context distribution to $p_w(c)\propto \bar w(c)\,p(c)$. 
For example, with EPIG, then it adjusts our original test-time input distribution $p(c)$ to increase the emphasis on inputs that have high expected weight under the context dependent prior $p(z|c)$.
Given one will generally construct the model by defining $p(c)$ and then $p(z|c)$, this highlights an important pitfall: simply adjusting the distribution $p(z|c)$ fails to account for the impact the weights have on contexts as well.

An important further nuance to these results occurs when the model is \emph{learned} from previous data using empirical risk minimisation, rather than being directly specified by the user.
Here one might consider incorporating the weight into the loss on which the model is trained instead of into the acquisition function.
That is, when training the model parameters, $\theta$, we could incorporate a $w(z)$ factor into the loss to induce the required change from $p(z)$ to $p_w(z)$.
While this will indeed actually have the desired effect on $p(z)$, it will also have the unwanted consequence of changing the distribution on $p(y|x,z,c)$.
Namely, by changing $p(\theta)$ to incorporate our desired weighting, this in turn changes $p(\theta|z,c)$, and in turn $p(y|x,z,c)$ (noting that in general $p_{\mathrm{pred}}(\mathrm{out}|\theta,\mathrm{in})$ will itself be fixed).
Thus, though this approach is in principle possible, it still requires the application of a weighting in the EPU (along with appropriately adjusting $p(c)$), this time to $y$ to account for the fact that our data distribution no longer represents our beliefs for what data we will actually see.

    \section{RELATED WORK}\label{sec:related_work}

Early decision-oriented work in the active-learning literature includes that of \citet{margineantu2005active}, who augmented standard predictive losses with input-dependent label-acquisition costs.
Similar considerations of variable labelling cost---as well as other practical factors, such as choosing between multiple possible label sources---were explored by \citet{donmez2008proactive}, \citet{kapoor2007selective}, \citet{nguyen2015combining} and \citet{werling2015onthejob}.
A bigger departure from standard practice was proposed by \citet{saartsechansky2007decision}, who shifted from predictive decision problems to more general problems.
Later work by \citet{javdani2014near}, \citet{sundin2019active} and \citet{filstroff2024targeted} was motivated by a similar emphasis on non-predictive decisions.
\citet{krishnamurthy2019active} focused on predictive decision problems but considered non-standard losses that place more weight on some label values than others.
More recently, \citet{hino2023active} used Bregman divergences to define new measures of committee disagreement in the context of query-by-committee data acquisition, while \citet{tan2023bayesian} used them to extend the ``mean cost of uncertainty'' framework \citep{yoon2013quantifying} in the particular case of classification.
All of the aforementioned work on active learning is similar to ours in its emphasis on non-standard losses, and some of it shares technical elements with our work, but (to our knowledge) our framework enables the derivation of practical acquisition objectives distinct from those in past work.

\begin{figure*}[t]
    \centering
    \includegraphics[width=\linewidth]{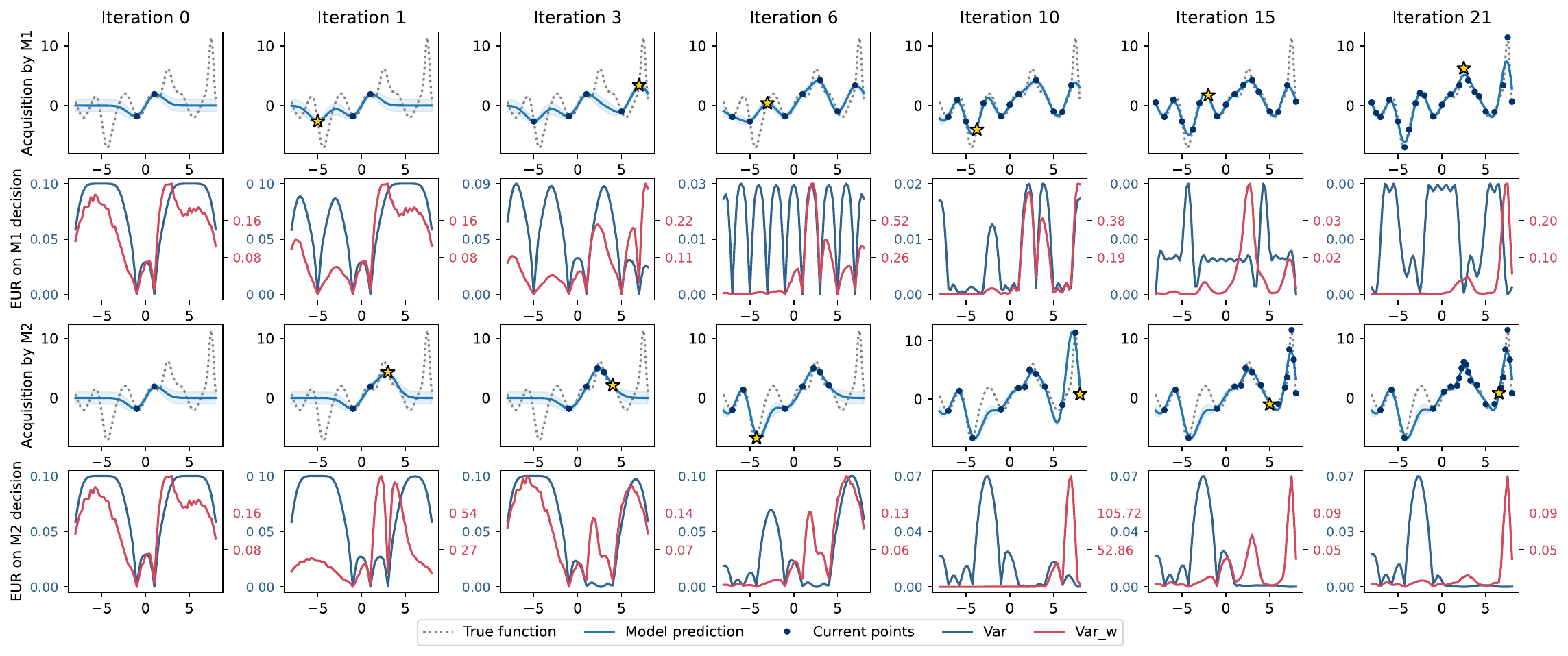}
    \vspace{-15pt}
    \caption{
        Expected variance reduction (\blue{\text{EVR}}) vs.\ a variant with weighting $w(z)=\exp(z)$ (\red{$\text{EVR}_w$}).
        Row~1: prediction with data acquired by \blue{\text{EVR}}.
        Row~2: values of \blue{\text{EVR}} and \red{$\text{EVR}_w$} given the Row~1 training set (the maximiser is labelled next).
        Row~3: prediction with data acquired by \red{$\text{EVR}_w$}.
        Row~4: values of \blue{\text{EVR}} and \red{$\text{EVR}_w$} given the Row~3 training set.
    }
    \vspace{-5pt}
    \label{fig:reg_vis}
\end{figure*}

Relevant foundational work in the experimental-design literature includes that of \citet{dawid1998coherent}, \citet{degroot1962uncertainty} and \citet{lindley1972bayesian}, who used Bayesian decision theory to derive acquisition objectives.
While principled, these approaches are generally computationally impractical. 
More recent work by \citet{huang2024amortized} demonstrated practical ideas in modern contexts by instead using amortisation to try and approximate the inner minimisation.
By contrast, our work avoids the need for amortisation or nested updating entirely by analytically minimising the Bregman divergence.

A complementary literature studies loss mismatches with an emphasis on learning from fixed data rather than acquiring new data.
For example, ``decision-focused learning'' trains predictors end-to-end through downstream (including combinatorial) optimisers to improve decision quality \citep{wilder2019melding}, and ``predict then optimise'' develops decision-calibrated surrogate losses with statistical guarantees for downstream optimisation \citep{elmachtoub2022smart}.
    \section{EXPERIMENTS}\label{sec:experiments}

Now we empirically assess the benefit of explicitly targeting the loss in a given predictive decision problem.
Specifically we run active learning on a range of regression and classification problems, in each case evaluating the performance of an acquisition objective that targets the test loss and comparing this against an alternative objective.
Code for reproducing our results is available at \shorturl{github.com/Zhuoyue-Huang/loss-driven-bayesian-active-learning}.
Additional practical details are in \Cref{sec:estimation,sec:experiment_details}.

\subsection{Regression}\label{sec:exp_reg}

We begin with four scenarios in which $z \in \bbR$ represents a real-valued output.
For each scenario, we consider two losses and their corresponding acquisition objectives (\Cref{sec:concrete_objectives}): standard squared error (standard EVR) and weighted squared error (weighted EVR).
Since one of these scenarios aligns with the original motivation for the Linex loss \citep{zellner1986bayesian}, we additionally explore the use of Linex on that problem.

\begin{figure}[t]
    \centering
    \includegraphics[width=\linewidth]{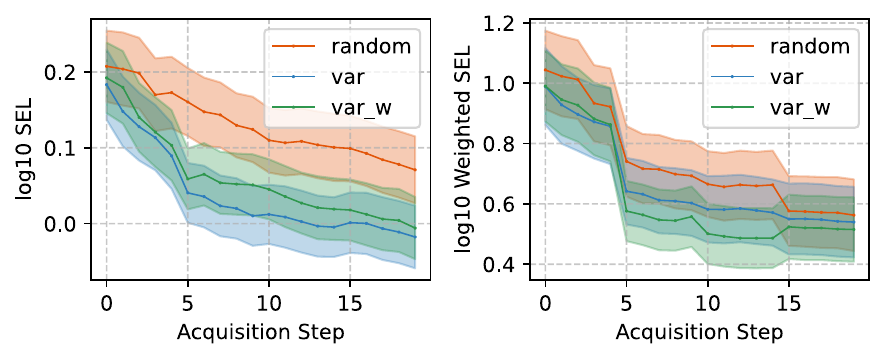}
    \includegraphics[width=\linewidth]{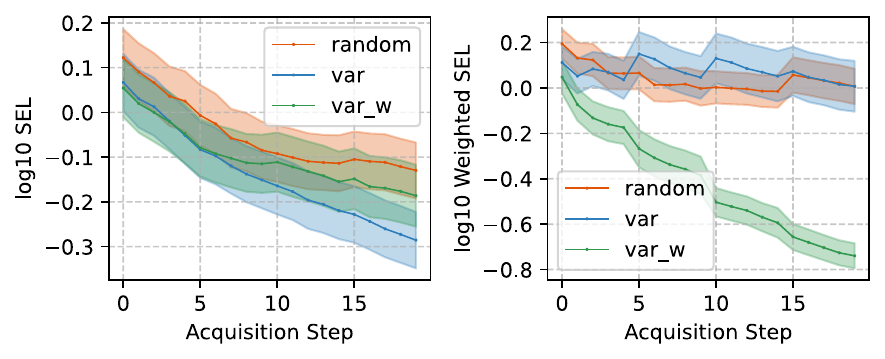}
    \includegraphics[width=\linewidth]{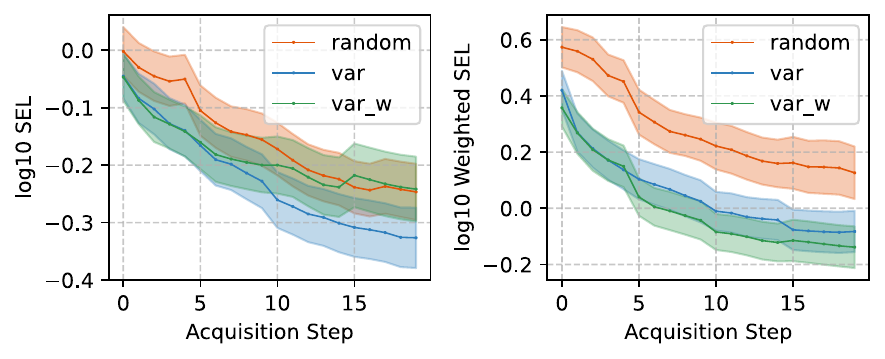}
    \caption{
        Top to bottom: performance (mean $\pm$ SEM) on \textsc{Slump}, \textsc{Yacht} and \textsc{Estate} with $w(z)=\exp(-z)$.
    }
    \label{fig:reg_metric}
\end{figure}

\subsubsection{Synthetic data}\label{sec:exp_reg_1d}

We begin with a scenario that can be easily visualised.
Specifically we consider one-dimensional inputs, $x \in \bbR$, and a true function, $f(x)$, defined as $f^* = 2f_0 + 8f_{(2.5, 0.5)}+ 10f_{(7.5, 0.25)} - 6f_{(-4.5, 0.5)}$
where $f_0(x)=\sin(2x)$ is a sine wave and $f_{(\mu, \sigma)}$ is a Gaussian density function parametrised by mean $\mu$ and standard deviation $\sigma$.
Observations follow $y=f(x)+\epsilon$ with $\epsilon\sim\mathcal N(0,0.04)$.
We consider $w(z) = \exp(z)$ as the weighting function of interest.
This corresponds to predictive errors being more consequential for higher values of the output, $z$, being predicted.

\begin{table}[t]
    \centering
    \small
    \begingroup
    \renewcommand{\SEM}[1]{ & \textnormal{\scriptsize #1}}

    \begin{tabular}{l
      r@{\,\textnormal{\scriptsize $\pm$\,}}l
      r@{\,\textnormal{\scriptsize $\pm$\,}}l
    }
        \toprule
        Method &
        \multicolumn{2}{c}{$\mathrm{SEL} \downarrow$} &
        \multicolumn{2}{c}{$\mathrm{SEL}_w \downarrow$}\\
        \midrule
        Random   & 0.9238 \SEM{0.0940} & 343.4 \SEM{109.5} \\
        EVR      & {\bfseries 0.3449} \SEM{0.0025} & 107.1 \SEM{1.971} \\
        $\text{EVR}_w$ & 1.5872 \SEM{0.1158} & {\bfseries 72.06}   \SEM{1.508} \\
        \bottomrule
    \end{tabular}
    \endgroup
    \caption{
        Test losses in 1d regression.
        Here $\mathrm{SEL}$ denotes squared-error loss and $\mathrm{SEL}_w$ denotes a weighted variant.
        For both test losses, the corresponding acquisition objective ($\mathrm{EVR}$ for the former; $\mathrm{EVR}_w$ for the latter) achieves the lowest loss.
        We report mean $\pm$ SEM over 25 runs.
        Bold indicates the best mean for each loss.
    }
    \vspace{-10pt}
    \label{tab:reg_metric}
\end{table}

\textbf{Model \& training}~
We use exact Gaussian-process (GP) regression \citep{williams2006gaussian} with zero mean, covariance $k(x,x')=\exp\!\left(-\tfrac{1}{2}(x-x')^2\right)$ and a Gaussian likelihood function that matches the true input-conditional output distribution.

\textbf{Data acquisition}~
In each run we start with 3 input-label pairs and acquire 25 additional pairs.
Candidates inputs, context inputs (for calculating EVR and weighted EVR) and test inputs are evenly spaced in $[-8,8]$, with 65, 49 and 97 points respectively.
With fixed hyperparameters, the GP one-step posterior update is exact, which we exploit to compute our acquisition objectives without refitting.

\textbf{Results}~
In \Cref{fig:reg_vis} we see how our choice of loss shapes data acquisition: the standard EVR just prioritises output regions in which the model's predictive variance is highest, while the weighted EVR additionally incorporates a prioritisation of high-output regions (in \Cref{app:reg_fig} we show how changing the weighting function leads to prioritising low-output regions).
In \Cref{tab:reg_metric} we see the importance of incorporating the test loss into data acquisition.
For both standard and weighted squared error, the acquisition objective that targets the loss achieves much better average loss across test data.

\subsubsection{UCI data}

Next we shift to three scenarios based on datasets from the UCI repository \citep{dua2017uci}: \textsc{Slump} ($N = 103$, $D = 7$; mix proportions $\rightarrow$ slump; \citealp{concrete_slump_test_182}), \textsc{Yacht} ($N = 308$, $D = 6$; hull descriptors $\rightarrow$ residuary resistance; \citealp{yacht_hydrodynamics_243}) and \textsc{Estate} ($N = 414$, $D = 6$; location/time features $\rightarrow$ price per unit area; \citealp{real_estate_valuation_477}); $N$ denotes the number of input-label pairs and $D$ the input dimensionality.
We consider $w(z) = \exp(-z)$ as the weighting function of interest, favouring accuracy in low-output regions; this is especially meaningful for \textsc{Yacht}, where low residuary resistance corresponds to lower required propulsive power and therefore improved performance.

\textbf{Model \& training}~
Again we use a GP model.
Here we use a constant prior mean, $m(x) = \mathrm{const}$.
The kernel is a sum of a linear (dot‑product) term and a Matérn term with i.i.d.\ white noise.
Additional modelling details are in \cref{app:gp_model_setting}.

\textbf{Data acquisition}~
In each run we start with 10 input-label pairs and acquire 20 additional pairs.
For reproducibility across methods, contexts and test inputs are drawn per dataset and resampled at each trial: we use fixed test sizes of $20$, $60$, and $80$ for \textsc{Slump}, \textsc{Yacht}, and \textsc{Estate}, respectively, with the rest being pool candidates.
We use the same acquisition methods as in \cref{sec:exp_reg_1d}.
Estimation details are in \cref{app:reg_estimation} and are independent of the kernel choice.

\textbf{Results}~
In \Cref{fig:reg_metric} we again see that the objective matched to the given test loss performs better than the unmatched objective.
In all data scenarios, $\text{EVR}_w$ attains the best $\mathrm{SEL}_w$ and $\text{EVR}$ attains the best $\mathrm{SEL}$.

\subsubsection{Linex loss}

So far we have compared EVR to its weighted variant.
Here we compare EVR to an alternative objective that targets a Linex loss, as presented in \Cref{tab:bregman_losses}.
Linex is asymmetric, penalising large overestimation errors much more strongly than underestimation for $\alpha>0$.
We reuse our setup for \textsc{Estate} from before but here acquire 50 new input-label pairs rather than 20.

\begin{figure}[t!]
    \centering
    \includegraphics[width=\linewidth]{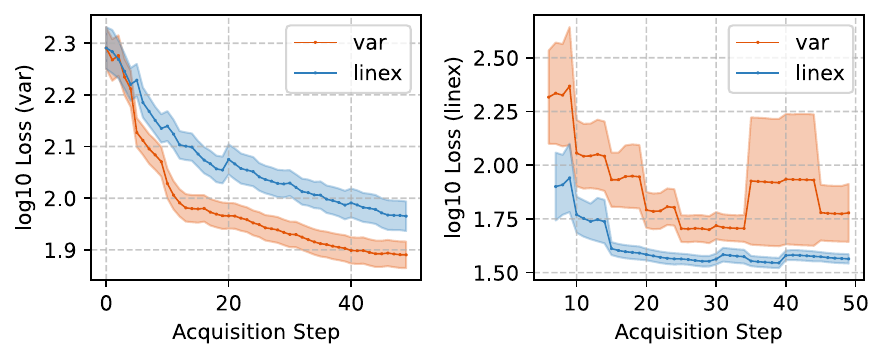}
    \caption{Evolution of performance (mean $\pm$ SEM) under downstream losses (SEL and Linex) on \textsc{Estate}.
    }
     \vspace{-5pt}
    \label{fig:linex_estate}
\end{figure}

\textbf{Results}~
In \Cref{fig:linex_estate} we see that the previously observed benefit of targeted data acquisition holds convincingly when the comparison is between standard squared error and the Linex loss.
Each acquisition objective is best under its own downstream loss, with the difference often constituting an order of magnitude.

\subsection{Classification}

Now we study three scenarios in which $z \in \{1, 2, \ldots, C\}$ is a class label using the additional UCI datasets \textsc{Vehicle} ($C=4$, $N=946$, $D=18$; silhouettes; \citealp{statlog_vehicle}), \textsc{Landsat} ($C=6$, $N=6435$, $D=36$; satellite imagery; \citealp{statlog_landsat}) and \textsc{Vowel} ($C=11$, $N=528$, $D=10$; speech features; \citealp{connectionist_bench_vowel}).
Two classes per dataset are designated high‑priority with $w(i)=50$; the remaining classes use $w(i)=1$.
For each scenario, we consider two losses and their corresponding acquisition objectives (\Cref{sec:concrete_objectives}): NLL (EPIG) and weighted NLL (weighted EPIG).

\textbf{Model \& training}~
We use random forests with 1{,}000 trees.
Given a forest $\theta$, each tree $j$ yields class probabilities $\theta_{j,i}(x)$ for class $i$, so that $p(z=i \mid x,\theta_j)=\theta_{j,i}(x)$ and the forest defines $p_w(z \mid x) \propto w(z) p(z \mid x)$ by averaging over trees then weighting with class importance.
The random forests are updated with standard (non‑Bayesian) training; the acquisition computations use the Bayesian predictive posterior.

\textbf{Data acquisition}~
In each run we start with 5 input-label pairs per class, after which we acquire 100 additional pairs.
For both acquisition-objective estimation and construction of the test set, we perform stratified subsampling with a fixed number per class: $45$ examples per class for \textsc{Vehicle} ($4\times 45 = 180$ total), $200$ per class for \textsc{Landsat} ($6\times 200 = 1{,}200$), and $15$ per class for \textsc{Vowel} ($11\times 15 = 165$).

\textbf{Results}~
With reference to the test losses in \Cref{fig:clf_metric} and class proportions in \Cref{fig:clf_proportion}, we can see that targeting the weighted loss leads us to prioritise high-weight classes, helping improve the test loss.
The performance benefit we saw in regression problems therefore carries over to classification problems.

\begin{figure}[t]
    \centering
    \includegraphics[width=\linewidth]{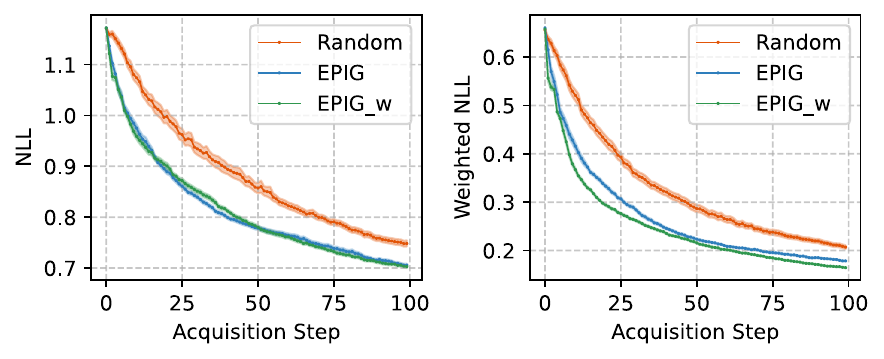}
    \includegraphics[width=\linewidth]{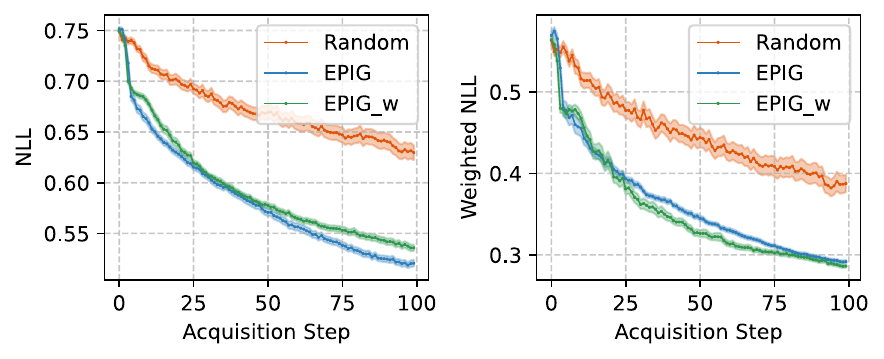}
    \includegraphics[width=\linewidth]{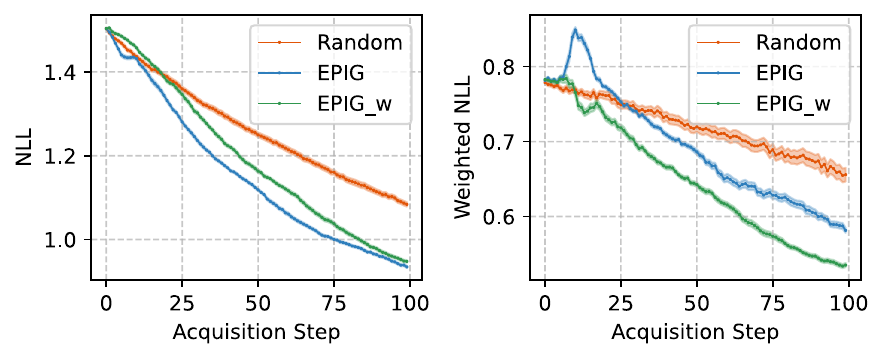}
    \caption{
        Top to bottom: mean $\pm$ SEM on \textsc{Vehicle} ($w=[50, 1, 1, 50]$), \textsc{Landsat} ($w=[1, 1, 1, 1, 50, 50]$) and \textsc{Vowel} ($w=[1, 1, 1, 1, 1, 1, 50, 50, 1, 1, 1]$).
    }
    \vspace{-5pt}
    \label{fig:clf_metric}
\end{figure}

    \section{CONCLUSION}
\looseness=-1
We have argued that popular active learning approaches do not generally account for the full range of decision problems that we want to target in practice.
To address this we have revisited the decision-theoretic foundations of Bayesian experimental design and identified a principled, general approach that directly targets the loss of interest.
We make this general approach practically applicable by showing that it can be analytically simplified when using weighted Bregman-divergence losses based on transformations of the world state.
These losses are very general, corresponding to the scenario where the Bayes act is a (weighted) posterior mean, which is true in many practical prediction and estimation problems.
We have further provided two example realisations of our framework, in the form of weighted variants of the EPIG and EVR acquisition functions that allow us to prioritise predictive performance over specific classes or regions.

    \section*{ACKNOWLEDGEMENTS}

ZH is supported by the EPSRC CDT in Statistics and Machine Learning (EP/Y034813/1).
FBS is supported by the EPSRC Probabilistic AI Hub (EP/Y028783/1).
TR is supported by the EPSRC grant EP/Y037200/1.
    \bibliography{references}

    \clearpage
    \section*{CHECKLIST}

\begin{enumerate}

  \item For all models and algorithms presented, check if you include:
  \begin{enumerate}
    \item A clear description of the mathematical setting, assumptions, algorithm, and/or model. [Yes]
    \item An analysis of the properties and complexity (time, space, sample size) of any algorithm. [Yes]
    \item (Optional) Anonymized source code, with specification of all dependencies, including external libraries. [Yes]
  \end{enumerate}

  \item For any theoretical claim, check if you include:
  \begin{enumerate}
    \item Statements of the full set of assumptions of all theoretical results. [Yes]
    \item Complete proofs of all theoretical results. [Yes]
    \item Clear explanations of any assumptions. [Yes]     
  \end{enumerate}

  \item For all figures and tables that present empirical results, check if you include:
  \begin{enumerate}
    \item The code, data, and instructions needed to reproduce the main experimental results (either in the supplemental material or as a URL). [Yes]
    \item All the training details (e.g., data splits, hyperparameters, how they were chosen). [Yes]
    \item A clear definition of the specific measure or statistics and error bars (e.g., with respect to the random seed after running experiments multiple times). [Yes]
    \item A description of the computing infrastructure used. (e.g., type of GPUs, internal cluster, or cloud provider). [Yes]
  \end{enumerate}

  \item If you are using existing assets (e.g., code, data, models) or curating/releasing new assets, check if you include:
  \begin{enumerate}
    \item Citations of the creator If your work uses existing assets. [Yes]
    \item The license information of the assets, if applicable. [Yes]
    \item New assets either in the supplemental material or as a URL, if applicable. [Yes]
    \item Information about consent from data providers/curators. [Yes]
    \item Discussion of sensible content if applicable, e.g., personally identifiable information or offensive content. [Yes]
  \end{enumerate}

  \item If you used crowdsourcing or conducted research with human subjects, check if you include:
  \begin{enumerate}
    \item The full text of instructions given to participants and screenshots. [Not Applicable]
    \item Descriptions of potential participant risks, with links to Institutional Review Board (IRB) approvals if applicable. [Not Applicable]
    \item The estimated hourly wage paid to participants and the total amount spent on participant compensation. [Not Applicable]
  \end{enumerate}

\end{enumerate}

    \appendix
    \onecolumn

    \section{FUNCTIONAL BREGMAN DIVERGENCES}\label{sec:functional_bregman}

For readability, the main text presents Bregman divergences and associated results in finite-dimensional vector notation. Many targets of interest in Bayesian active learning are more naturally viewed as elements in an infinite-dimensional function space (e.g., probability densities or latent functions).
This appendix summarises \emph{functional} Bregman divergences and records the functional analogue of the key optimality and uncertainty identities used in the paper.

Let $(\mathcal{X}, \Sigma, \mu)$ be a $\sigma$-finite measure space and let $L^p := L^p(\mathcal{X},\mu)$
for some $1 \le p < \infty$. Let $\mathcal{K} \subset L^p$ be an open convex set and let
$\Phi:\mathcal{K}\to\mathbb{R}$ be strictly convex and G\^{a}teaux differentiable on $\mathcal{K}$.
For $g\in\mathcal{K}$ and direction $h\in L^p$, write the (G\^{a}teaux) derivative as
\[
\delta\Phi(g;h) := \lim_{\epsilon\to 0} \frac{\Phi(g+\epsilon h)-\Phi(g)}{\epsilon},
\]
whenever the limit exists \citep{hiriart2004fundamentals}.

\begin{definition}[Functional Bregman divergence \citep{frigyik2008functional}]
For $f,g\in\mathcal{K}$, the functional Bregman divergence induced by $\Phi$ is
\[
D_{\Phi}(f,g) := \Phi(f) - \Phi(g) - \delta\Phi(g; f-g).
\]
\end{definition}

A mild regularity condition is that for each fixed $g\in\mathcal{K}$ the map
$h\mapsto \delta\Phi(g;h)$ defines a bounded linear functional on $L^p$ \citep{frigyik2008functional}. By $L^p$--$L^q$ duality, there exists
$\phi_g \in L^q$ (with $1/p + 1/q = 1$, and $q=\infty$ if $p=1$) such that
\[
\delta\Phi(g;h) = \int_{\mathcal{X}} \varphi_g(x)\,h(x)\,d\mu(x)
\quad\text{for all } h\in L^p.
\]
In this case,
\begin{equation}
D_{\Phi}(f,g) \;=\; \Phi(f)-\Phi(g)-\langle \varphi_g, f-g\rangle,
\qquad
\langle \varphi_g, f-g\rangle := \int_{\mathcal{X}} \varphi_g(x)\,(f(x)-g(x))\,d\mu(x).
\label{eq:func_inner_prod}
\end{equation}

Throughout we assume this mild regularity so that $\delta\Phi(g;\cdot)\in (L^p)^*$ and therefore admits the integral representation above, allowing us to identify $\varphi_g$ with the (Gâteaux/Fréchet) gradient of $\Phi$ at $g$; for a convex $\Phi$ this coincides with the unique subgradient at $g$ when differentiability holds. Readers may safely think of $\varphi_g$ as the “gradient of $\Phi$ at $g$” for intuition.

All results in the paper stated for $D_\phi$ can be extended to functional Bregman divergence $D_\Phi$ under the replacements $\phi\mapsto\Phi$, $\nabla\phi(\cdot)\mapsto \varphi_{(\cdot)}$ (or $\delta\Phi(\cdot;\cdot)$), and $\langle\cdot,\cdot\rangle\mapsto \int \cdot\,\cdot\,d\mu$.
In particular, the Bayes act under $D_\Phi(F,g)$ for a random function $F$ is still the mean function: under mild integrability and assuming $\mathbb{E}[F]\in\mathcal K$,
\[
\arg\min_{g\in\mathcal K}\mathbb{E}\big[D_\Phi(F,g)\big] = \mathbb{E}[F],
\qquad
\min_{g\in\mathcal K}\mathbb{E}\big[D_\Phi(F,g)\big]
= \mathbb{E}[\Phi(F)]-\Phi(\mathbb{E}[F]),
\]
a nonnegative Jensen gap \citep{frigyik2008functional}.

\paragraph{Example: KL divergence between densities.}
Let $\mathcal K$ be a convex set of strictly positive probability densities on $\mathcal X$ (w.r.t. $\mu$) and define $\Phi(p):=\int_{\mathcal X} p(x)\log p(x)\,d\mu(x)$. Then the induced functional Bregman divergence equals
\[
D_\Phi(p,q) = \int_{\mathcal X} p(x)\log\frac{p(x)}{q(x)}\,d\mu(x)=\mathrm{KL}(p\|q),
\]
where the normalisation $\int p\,d\mu=\int q\,d\mu=1$ ensures cancellation of additive terms
\citep{frigyik2008functional}.
    \section{PROPER SCORING RULES AND DIVERGENCE REPRESENTATION}\label{app:proper}

When the prediction target is itself a probability distribution (e.g., in probabilistic learning tasks), losses are naturally expressed as \emph{scoring rules}.
This appendix records the standard link between (strictly) proper scoring rules and (functional) Bregman divergences.
This connection motivates our focus on Bregman-type losses: many commonly used probabilistic losses, including log loss, arise as proper scoring rules whose regret is a Bregman divergence.

Let $(Z,\mathcal{Z})$ be a measurable space, $\mathcal{P}$ a convex class of probability laws on $Z$, and $S:\mathcal{P}\times Z\to\mathbb{R}$ a (loss-oriented) scoring rule. $S$ is \emph{(strictly) proper} if for all $p,q\in\mathcal{P}$,
\begin{equation*}
    \mathbb{E}_{p(z)}\!\left[S(p,z)\right]\le \mathbb{E}_{p(z)}\!\left[S(q,z)\right]
\end{equation*}
with equality if and only if $q=p$. A differentiable scoring rule is strictly proper if and only if its \emph{score regret} is a (functional) Bregman divergence \citep{gneiting2007strictly,ovcharov2018proper}:
\begin{equation}
\label{eq:score_regret_equals_bregman}
\expectation{p(z)}{S(q,z)}-\expectation{p(z)}{S(p,z)}\!=\!
\begin{cases}
D_{\phi}(p,q), \hspace{-0.5em}& p,q\in\Delta^{K-1}\\
D_{\Phi}(p,q), \hspace{-0.5em}& p,q\in\mathcal{P}.
\end{cases}
\end{equation}
In finite-outcome problems this yields the \emph{divergence representation} \citep{gneiting2007strictly}:
\begin{align}
    S(q,z)&= D_\phi(e_z,q) + c(z)\nonumber\\
          &= -\phi(q) - \langle \nabla\phi(q),\,e_z-q\rangle + c(z)\label{eq:dis-proper-score},
\end{align}
where scores are equivalent up to an additive $c(z)$. On general outcome spaces, the \emph{Savage representation} \citep{gneiting2007strictly, dawid2014theory} states that there exists a concave entropy $H$ on $\mathcal{P}$ and a subgradient $h_q\in\partial H(q)$ such that
\begin{equation*}
    S(q,z)=H(q)+h_q(z)-\int h_q\,dq +c(z).
\end{equation*}
Writing $\Phi:=-H$ (so $\Phi$ is convex) and choosing any $\varphi_q\in\partial\Phi(q)$ with $\varphi_q=-h_q$, we obtain the continuous analogue\footnote{To match with the Bregman divergence formulation in \cref{eq:func_inner_prod}, intuitively $\langle \varphi_q,\delta_z-q\rangle\coloneqq\varphi_q(z)-\int\varphi_q\,dq$. This extension sits slightly outside the $L^p$–$L^q$ pairing used for functional Bregman divergences but is consistent with the general function–measure duality \citep{dawid2007geometry}.} of the discrete formula:
\begin{align}
S(q,z)= -\Phi(q) - \varphi_q(z) + \expectation{q(z)}{\varphi_q(z)} + c(z),\label{eq:cts-proper-score}
\end{align}
and consequently the score regret equals the functional Bregman divergence $D_\Phi$ \citep{ovcharov2018proper}.
\cref{eq:score_regret_equals_bregman} makes clear that, when the inputs are \emph{distributions}, one uses the appropriate Bregman geometry for the input type: vector-input $D_{\phi}$ on the simplex (discrete) and functional $D_{\Phi}$ on densities (continuous). Minimising expected score under the truth $p$ is therefore equivalent to minimising (functional) Bregman divergence to $p$, uniquely solved by $q=p$. 

\paragraph{Example: KL divergence between densities.} The log score $S(q,z)=-\log q(z)$ is strictly proper and its regret is $\mathrm{KL}(p\|q)$.
Equivalently, it corresponds to the convex functional $\Phi(p)=\int p\log p$ in \cref{sec:functional_bregman}.

Note that we restrict the discussion to cases where the decision rule outputs a predictive distribution $q$, rather than a point prediction. The primitive loss is then a scoring rule $S(q,z)$. In this setting, the Bregman divergence appears at the level of expected score regret. It should therefore be distinguished from the pointwise Bregman loss $D_\phi(T(z),a)$ used in the main text for point or embedded-point prediction. 

In the case of discrete probability distribution spaces, one can bridge between Bregman divergence for point prediction and distribution directly. If $z\in\{1,\dots, K\}$ and $e_z\in\Delta^{K-1}$ is the one-hot encoding, then $\expectation{p_{n}(z)}{e_z}=p_{n}(z)$, the full class-probability vector. By the same proof of evaluation-discrepancy identity in \cref{prop:eval-decomp}, 
\begin{equation*}
    \expectation{p_{n}(z)}{D_\phi(e_z, q)} - \expectation{p_{n}(z)}{D_\phi(e_z, p_n(z))}=D_\phi(p_n,q)
\end{equation*}
so the expected pointwise Bregman loss on one-hot outcomes induces the Bregman divergence on the simplex under the same potential. Therefore, when the targets are in a discrete probability space, the optimal data-acquisition policy is equivalent to the minimiser of the expected proper scoring rule $S$ (defined previously in \cref{eq:dis-proper-score,eq:cts-proper-score}):
\begin{equation*}
    \pi^\ast_d = 
    \underset{\pi_d\in\Pi_d}{\arg\min}\,\expectation{p_n(z,d;\pi_d)}{S(p_n(\cdot|d;\pi_d), z)}.
\end{equation*}
    \section{PROOFS AND DERIVATIONS}\label{sec:proofs}
\subsection{Generalised entropy and expected predictive uncertainty}\label{app:proof_uncertainty_epu_general_form}

\begin{proposition}[\citealp{banerjee2005optimality}]\label{thm:expected_bregman_minimiser}
    Let $\mathcal Z \subseteq \bbR^K$ be an open convex set, $\phi: \mathcal Z \to \bbR$ be a strictly convex differentiable potential function, and $z \sim p(z)$ a random variable taking values in $\mathcal Z$ for which both $\expectation{p(z)}{z}$ and $\expectation{p(z)}{\phi(z)}$ are finite.
    With the assumption $\expectation{p(z)}{z}\in\mathcal Z$, the minimiser of the expected Bregman divergence is the expectation of $z$:
    \begin{align*}
        \argmin_{a \in \mathcal Z} \expectation{p(z)}{D_\phi(z, a)}
        =
        \expectation{p(z)}{z}
        .
    \end{align*}
\end{proposition}

\begin{theorem*}
\looseness=-1
Assume the terminal loss takes the form of a weighted Bregman divergence on measurable transformations of the world state as per~\Cref{eq:weighted_bregman_loss}, and that $\expectation{q(z)}{w(z)}$, $\expectation{q(z)}{w(z)T(z)}$ and $\expectation{q(z)}{w(z)\phi(T(z))}$ are finite for some generic distribution $q(z)$. If $\mathcal{A} \subseteq \mathrm{ri}(\mathrm{dom}(\phi))$ is convex and contains $\expectation{q_w(z)}{T(z)}$ where  $q_w(z) \!=\! w(z) q(z) / \bar{w}_q$ and $\bar{w}_q \!=\! \expectation{q(z)}{w(z)}$, then the generalised entropy associated with this $q$ is given by
\begin{align*}
    \uncertainty{\phi,T}{w}{q(z)} := \min_{a \in \calA} \expectation{q(z)}{w(z) D_\phi(T(z), a)}=
    \bar{w}_q \left( \expectation{q_w(z)}{\phi(T(z))} - \phi(\expectation{q_w(z)}{T(z)}) \right).
\end{align*}
Assuming that $p(z|d;\pi_d)$ satisfies the above restrictions of $q$ for all $(d,\pi_d)$,
then the Bayes-optimal data gathering policy for our model-loss pairing is given by
\begin{align*}
    \pi_d^* = \argmin_{\pi_d\in\Pi_d} ~\bar{w}\,\mathbb{E}_{p_w(d;\pi_d)}\Big[
        \expectation{p_w(z | d;\pi_d)}{\phi(T(z))} - \phi(\expectation{p_w(z | d;\pi_d)}{T(z)}),
\end{align*}
where we define beliefs $p_w(d;\pi_d)=\bar w(d,\pi_d)p(d ; \pi_d) /\bar w$ and $p_w(z | d ; \pi_d)  =  w(z) p(z | d ; \pi_d) / \bar{w}(d, \pi_d)$, and weights $\bar{w}(d,\pi_d)  =  \expectation{p(z | d ; \pi_d)}{w(z)}$ and $\bar w = \expectation{p(z)}{w(z)}$.
\end{theorem*}
\begin{proof}
    With a change of belief, we have
    \begin{align*}
        \expectation{q(z)}{w(z) D_\phi(T(z), a)}
        =
        \bar{w}_q\,\expectation{q_w(z)}{D_\phi(T(z), a)}.
    \end{align*}
    As $\bar{w}$ is a constant, by \Cref{thm:expected_bregman_minimiser}, the minimiser of $\bar{w}_q\,\expectation{q_w(z)}{D_\phi(T(z), a)}$ is $a^*=\expectation{q_w(z)}{T(z)}$. Therefore, the corresponding minimal value is
    \begin{align*}
        \uncertainty{\phi}{w}{q(z)}&=\bar{w}_q\expectation{q_w(z)}{D_\phi(T(z),a^*)}\\
        &=
        \bar{w}_q\left(\expectation{q_w(z)}{\phi(T(z)) - \phi(a^*) - \langle \nabla \phi(a^*), T(z) - a^* \rangle}\right)
        \\
        &=
        \bar{w}_q\left(\expectation{q_w(z)}{\phi(T(z))} - \phi(a^*) - \expectation{q_w(z)}{\langle \nabla \phi(a^*), T(z) \rangle} + \langle \nabla \phi(a^*), a^* \rangle\right)
        \\
        &=
        \bar{w}_q\left(\expectation{q_w(z)}{\phi(T(z))} - \phi(a^*)\right).
    \end{align*}
    Under the predictive distribution $p(z|d;\pi_d)$, the expected predictive uncertainty \eqref{eq:data_objective_explicit_min} can be expressed as
    \begin{align*}
        \EPU_{p(d, z; \pi_d)}^{\phi,T,w} &= \expectation{p(d;\pi_d)}{h_{\phi}^w[p(z|d;\pi_d)]}\\
        &=\expectation{p(d;\pi_d)}{\bar{w}(d,\pi_d)\left(\expectation{p_w(z|d;\pi_d)}{\phi(T(z))} - \phi(\expectation{p_w(z|d;\pi_d)}{T(z)})\right)}\\
        &=\bar{w} \,\mathbb{E}_{p_w(d;\pi_d)}\Big[
        \expectation{p_w(z | d;\pi_d)}{\phi(T(z))} - \phi(\expectation{p_w(z | d;\pi_d)}{T(z)}) \Big],
    \end{align*}
    which leads to the optimal policy $\pi_d^* = \argmin_{\pi_d\in\Pi_d}\EPU_{p(d, z; \pi_d)}^{\phi,T,w}$ as the form required.
\end{proof}

\subsection{Expected posterior uncertainty with weighted divergence}\label{app:weighted_EPU_proof}
\begin{corollary*}
    Assume the generalised entropy terms are well defined. Then, with $\bar w(c)=\expectation{p(z|c)}{w(z)}$ and $p_w(z|c)=w(z)\,p(z|c) / \bar w(c)$,
    \begin{align*}
    \mathrm{EPU}^{\phi,T,w}_{p(z,y|c,x)}=\bar w(c)\,
    \mathrm{EPU}^{\phi,T}_{p_w(z| c)\,p(y| x,z,c)},
    \end{align*}
    where the lack of the $w$ superscript in the second $\mathrm{EPU}$ is used to imply $w(z)=1, \forall z$.
    Moreover, averaging over contexts gives
    \begin{align*}
    \mathrm{EPU}^{\phi,T,w}_{p(c,y,z|x)}&=\bar w\,\mathrm{EPU}^{\phi, T}_{p_w(c)\,p_w(z| c)\,p(y| x,z,c)}, \\
    &= \bar w\,\mathrm{EPU}^{\phi, T}_{p_w(z)\,p(c,y| x,z)},
    \end{align*}
    where $\bar w=\mathbb{E}_{p(c)}[\bar w(c)]$ and $p_w(c)=\bar w(c)\,p(c)/\bar w$. The same identities hold with $\mathrm{EPU}$ replaced by $\mathrm{EUR}$.
\end{corollary*}
\begin{proof}
    The expected posterior uncertainty with weighted divergence can be expanded as
\begin{align*}
    \mathrm{EPU}^{\phi,T,w}_{p(z,y\mid c,x)}&=
    \mathbb{E}_{p(y\mid x,c)}\bigl[
    h^w_{\phi,T}\bigl[p(z \mid c,y,x)\bigr]\bigr]\\
    &=\E_{p(y\mid x, c)\,p(z \mid c, y,x)}\left[w(z)D_\phi\left(T(z),\E_{p_w(z \mid c,y,x)}[T(z)]\right)\right]\\
    &=\E_{p(z\mid c)\,p(y|x, z, c)}\left[w(z)D_\phi\left(T(z),\E_{p_w(z \mid c,y,x)}[T(z)]\right)\right]\\
    &=\E_{p(z\mid c)}[w(z)]\cdot\E_{p_w(z\mid c)\,p(y|x, z, c)}\left[D_\phi\left(T(z),\E_{p_w(z\mid c,y,x)}[T(z)]\right)\right].
\end{align*}
Note that $p_w(z\mid c,y,x)$ can also be expressed by the Bayes rule,
\begin{align}
    p_w(z\mid c,y,x)
    &=\frac{w(z)}{\mathbb{E}_{p(z\mid c,y,x)}\bigl[w(z)\bigr]}
    \cdot p(z\mid c,y,x)\nonumber
    \\
    &=\frac{w(z)}
            {\mathbb{E}_{p(z\mid c)}\!\Bigl[\,
            \frac{p(y\mid x,z,c)}
                    { \mathbb{E}_{p(z'\mid c)}\bigl[p(y\mid x,z',c)\bigr]}
            \,w(z)\Bigr]}
    \cdot
    \frac{p(z\mid c)\,p(y\mid x,z,c)}
        {\mathbb{E}_{p(z'\mid c)}\bigl[p(y\mid x,z',c)\bigr]}\nonumber
    \\
    &=\frac{w(z)}
            {\mathbb{E}_{p(z\mid c)}\!\Bigl[\,
            p(y\mid x,z,c)
            \,w(z)\Bigr]}
    \cdot
    p(z\mid c)\,p(y\mid x,z,c)\nonumber
    \\
    &=\frac{w(z)\,p(z\mid c)}
            {\mathbb{E}_{p(z\mid c)}\bigl[w(z)\bigr]}
    \cdot
    \frac{p(y\mid x,z,c)}
        {\mathbb{E}_{p_w(z\mid c)}\!\bigl[p(y\mid x,z,c)\bigr]}\nonumber\\
    &=p_w(z\mid c)
    \cdot
    \frac{p(y\mid x,z,c)}
        {\mathbb{E}_{p_w(z\mid c)}\!\bigl[p(y\mid x,z,c)\bigr]},\label{eq:rewighted_bayes_rule}
\end{align}
where the reweighted marginal likelihood over $y$ is shown in the form of
\begin{align*}
    q(y\mid x, c)\coloneqq\E_{p_w(z\mid c)}[p(y\mid x, z, c)].
\end{align*}
In this case, we can further simplify the EPU as
\begin{align*}
    \mathrm{EPU}^{\phi,T,w}_{p(z,y\mid c,x)}&=\E_{p(z\mid c)}[w(z)]\cdot\E_{q(y\mid x, c)\,p_w(z\mid c, y,x)}\left[D_\phi\left(T(z),\E_{p_w(z\mid c, y, x)}[T(z)]\right)\right]\\
    &=\E_{p(z\mid c)}[w(z)]\cdot\E_{q(y\mid x, c)}\left[h_{\phi, T}\left[p_w(z\mid c, y, x)\right]\right]\\
    &=\E_{p(z\mid c)}[w(z)]\cdot\text{EPU}^{\phi,T}_{q(y\mid x, c)\,p_w(z\mid c, y,x)}\\
    &=\bar w(c)\,\text{EPU}^{\phi,T}_{p_w(z\mid c)\,p(y\mid x, z, c)}
    .
\end{align*}
Averaging over contexts gives
\begin{align*}
    \mathrm{EPU}^{\phi,T,w}_{p(z,y,c|x)}&=\expectation{p(c)}{\bar w(c)\cdot\text{EPU}^{\phi,T}_{p_w(z\mid c)\,(y\mid x, z, c)}}\\
    &=\expectation{p(c)}{\bar w(c)}\cdot\expectation{p(c)}{\frac{\bar w(c)}{\expectation{p(c)}{\bar w(c)}}\cdot\text{EPU}^{\phi,T}_{p_w(z\mid c)\,p(y\mid x, z, c)}}\\
    &=\expectation{p(c)}{\bar w(c)}\cdot\expectation{p_w(c)}{\text{EPU}^{\phi,T}_{p_w(z\mid c)\,p(y\mid x, z, c)}}\\
    &=\bar w\,\mathrm{EPU}^{\phi,T}_{p_w(c)\,p_w(z| c)\,p(y| x,z,c)},
    \end{align*}
    where the joint weighted distribution can also be expressed as
    \begin{align*}
        p_w(c)\,p_w(z| c)\,p(y| x,z,c)&=\frac{\bar w(c)}{\bar w}\,p(c)\cdot\frac{w(z)}{\bar w(c)}\,p(z|c)\cdot p(y| x,z,c)\\
        &=\frac{w(z)}{\expectation{p(c)\,p(z|c)}{w(z)}}\,p(z)\,p(c|z)\,p(y| x,z,c)\\
        &=p_w(z)\,p(y,c|x,z).
    \end{align*}
    As $\uncertainty{\phi,T}{w}{p(z | c)}=\bar w(c) \uncertainty{\phi,T}{}{p_w(z | c)}$ and $\expectation{p(c)}{\uncertainty{\phi,T}{w}{p(z | c)}} = \bar w\,\expectation{p_w(c)}{\uncertainty{\phi,T}{}{p_w(z | c)}}$, we have that
    \begin{align*}
        \EUR_{p(y, z | x, c)}^{\phi,T,w} &= \uncertainty{\phi,T}{w}{p(z | c)} - \EPU_{p(y, z | x, c)}^{\phi,T,w} \\
        &= \bar w(c)\left(\uncertainty{\phi,T}{}{p_w(z | c)} - \text{EPU}^{\phi,T}_{p_w(z\mid c)\,p(y\mid x, z, c)}\right) \\
        &= \bar w(c)\,\text{EUR}^{\phi,T}_{p_w(z\mid c)\,p(y\mid x, z, c)}\\
        \EUR_{p(c, y, z | x)}^{\phi,T,w} &= \expectation{p(c)}{\uncertainty{\phi,T}{w}{p(z | c)}} - \EPU_{p(c, y, z | x)}^{\phi,T,w} \\
        &= \bar w\left(\expectation{p_w(c)}{\uncertainty{\phi,T}{}{p_w(z | c)}}-\mathrm{EPU}^{\phi,T}_{p_w(c)\,p_w(z| c)\,p(y| x,z,c)}\right)\\
        &=\bar w \, \mathrm{EUR}^{\phi,T}_{p_w(c)\,p_w(z| c)\,p(y| x,z,c)}\\
        &=\bar w \, \mathrm{EUR}^{\phi,T}_{p_w(z)\,p(y,c|x,z)}.\qedhere
    \end{align*}
\end{proof}

\subsection{Decomposition of evaluation discrepancy}\label{app:decomp_uncertainty_proof}
In practice we often evaluate a model against an \emph{external} system $p_{\mathrm{eval}}(z)$ (a data source, simulator, sensor, or benchmark), which need not coincide with the model's predictive $p(z)$.
Following the decision-theoretic view of externally grounded evaluation, the relevant score is the expected loss of the Bayes action under $p$, taken with respect to $p_{\mathrm{eval}}$ \citep[Sec.~5.4]{bickfordsmith2025rethinking}.
For Bregman losses, this score admits a clean two-term decomposition.

\begin{proposition}[Evaluation-discrepancy decomposition under Bregman loss]\label{prop:eval-decomp}
    Let $D_\phi$ be the Bregman divergence induced by a strictly convex differentiable $\phi$.
    Then
    \begin{align*}
        &\mathbb{E}_{p_{\mathrm{eval}}}\!\left[D_\phi\!\left(z,\,\mathbb{E}_{p}[z]\right)\right]=
    \underbrace{D_\phi\!\left(\mathbb{E}_{p_{\mathrm{eval}}}[z],\,\mathbb{E}_{p}[z]\right)}_{\mathrm{estimation\ error}}
    +    \underbrace{\mathbb{E}_{p_{\mathrm{eval}}}\!\left[D_\phi\!\left(z,\,\mathbb{E}_{p_{\mathrm{eval}}}[z]\right)\right]}_{\mathrm{irreducible\ dispersion}}.
    \end{align*}
\end{proposition}

The left-hand side of \Cref{prop:eval-decomp} (see \cref{app:decomp_uncertainty_proof} for a proof) is the \emph{expected Bregman loss} of acting with the model's Bayes action when reality is generated by $p_{\mathrm{eval}}$.
The right-hand side separates:
(i) a \emph{reducible} term—how far the model's Bayes act $\mathbb{E}_{p}[z]$ is from the evaluation Bayes act $\mathbb{E}_{p_{\mathrm{eval}}}[z]$ in the geometry of $\phi$; and
(ii) an \emph{irreducible} term—the Jensen gap of $p_{\mathrm{eval}}$ under $\phi$, i.e., the data dispersion intrinsic to the evaluation source.
Thus, improvement under an external metric comes \emph{only} by shrinking the first term; the second is fixed by $p_{\mathrm{eval}}$ and the chosen loss.
This generalises the evaluation identities in \citet{bickfordsmith2025rethinking} from squared and log loss to \emph{any} Bregman divergence, unifying bias-variance and cross-entropy decompositions within Bregman-geometric framework within one decision-theoretic statement.
\begin{proof}

Following \citet{bickfordsmith2025rethinking}, define the evaluation discrepancy
\begin{align*}
    d(p,p_{\text{eval}})=\E_{p_{\text{eval}}}[\ell(z,a^\ast)-\ell(z,a_{\text{eval}})],
\end{align*}
where $a^\ast$ and $a_{\text{eval}}$ are the minimisers of expected loss over $p$ and $p_{\text{eval}}$, respectively. When the loss is in the form of a Bregman divergence, we can further reduce the expression as
\begin{align*}
    d(p,p_{\text{eval}})&=\E_{p_{\text{eval}}}\left[D_\phi\left(z,\E_{p}(z)\right)-D_\phi\left(z,\E_{p_{\text{eval}}}(z)\right)\right]\\
    &=\E_{p_{\text{eval}}}\left[\phi(z) - \phi\left(\E_{p(z)}[z]\right)-\left\langle\nabla\phi\left(\E_{p(z)}[z]\right),z-\E_{p(z)}[z]\right\rangle\right]-\left(\E_{p_{\text{eval}}}\left[\phi(z)\right] - \phi\left(\E_{p_{\text{eval}}}[z]\right)\right)\\
    &=\E_{p_{\text{eval}}}\left[\phi(z)\right]-\phi\left(\E_{p(z)}[z]\right)-\left\langle\nabla\phi\left(\E_{p(z)}[z]\right),\E_{p_{\text{eval}}}[z]-\E_{p(z)}[z]\right\rangle-\left(\E_{p_{\text{eval}}}\left[\phi(z)\right] - \phi\left(\E_{p_{\text{eval}}}[z]\right)\right)\\
    &=D_\phi\left(\E_{p_{\text{eval}}}[z],\E_{p(z)}[z]\right).
\end{align*}
Rearrange the first and the last equation, we have that
\begin{align*}
    \E_{p_{\text{eval}}}\left[D_\phi\left(z,\E_{p}(z)\right)\right] = D_\phi\left(\E_{p_{\text{eval}}}[z],\E_{p(z)}[z]\right) + \underbrace{\E_{p_{\text{eval}}}\left[D_\phi\left(z,\E_{p_{\text{eval}}}(z)\right)\right]}_{\text{irreducible}},
\end{align*}
which aligns with the bias-variance decomposition in Bregman divergence (\cref{eq:bv-decomp}). For any (possibly random) predictor $A$, define the \emph{central prediction}
$\mathcal{E}_\phi[A]\!:=\!\arg\min_{u}\,\mathbb{E}_A[D_\phi(u,A)]$.
Then (e.g., \citealp{pfau2013generalised})
\begin{align}
  \mathbb{E}_{Z,A}\!\left[D_\phi(Z,A)\right]
  &= \mathbb{E}_Z\!\left[D_\phi\!\left(Z,\mathbb{E}[Z]\right)\right]
   + \mathbb{E}_A\!\left[D_\phi\!\left(\mathbb{E}[Z],A\right)\right] \nonumber\\
  &= \underbrace{\mathbb{E}_Z\!\left[D_\phi\!\left(Z,\mathbb{E}[Z]\right)\right]}_{\text{Bayes error}} + \underbrace{D_\phi\!\left(\mathbb{E}[Z],\mathcal{E}_\phi[A]\right)}_{\text{bias}}+ \underbrace{\mathbb{E}_A\!\left[D_\phi\!\left(\mathcal{E}_\phi[A],A\right)\right]}_{\text{model variance}}\label{eq:bv-decomp}.
\end{align}
\end{proof}

\subsection{Uncertainty-based acquisition via Bregman divergences recovers predictive-variance reduction, BALD, and EPIG}\label{app:epu_retrive}

\paragraph{Predictive posterior variance \citep{cohn1994active}}

Set the Bregman potential to the quadratic $\phi(u)=u^2$, take the quantity of interest to be the latent response $z=f(c)$ at a context $c$, and let the context distribution be the input distribution $p(c)$.
Then the Bregman uncertainty is the predictive variance,
$h_\phi[p(z\mid c)]=\text{Var}_n\!\big(z\mid c\big)$,
and the pool-based EPU at candidate $x$ is
\begin{align*}
    \mathrm{EPU}(x)
    =\E_{p(c)\,p(y\mid x,c)}\big[\text{Var}_{p(z\mid c, y,x)}[z]\big],
\end{align*}
which retrieves a Bayesian variant of the objective in \cite{cohn1994active}.

\paragraph{BALD \citep{houlsby2011bayesian}}

Choose the log score (entropy) potential $\phi$ and set the target to the model parameters $z=\theta$.
The pool-based EUR becomes the expected drop in entropy of $\theta$,
\begin{align*}
    \mathrm{EUR}(x)
    =\E_{p(y\mid x)}\!\big[\,\mathrm H\big[p(\theta)\big]-\mathrm H\big[p(\theta\mid x,y)\big]\big]
    = \mathrm I(\theta; y\mid x),
\end{align*}
which is the BALD mutual-information objective \cite{bickfordsmith2023prediction}.

\paragraph{EPIG \citep{bickfordsmith2023prediction}}

With the same entropy potential, take the target-context pair $(z,c)=(y_\ast,x_\ast)$ and average the context-aware EUR over the target-input distribution $p(x_\ast)$ (cf. \cref{sec:method}):
\begin{align*}
    \E_{p(x_\ast)}[\mathrm{EUR}(x;x_\ast)]=
    \mathbb{E}_{p(x_\ast)\,p(y\mid x)}\bigl[\mathrm{H}\bigl[p(y_\ast \mid x_\ast)\bigr]
    - \mathrm{H}\bigl[p(y_\ast \mid x_\ast, x,y)\bigr]\bigr],
\end{align*}
which is precisely the Expected Predictive Information Gain (EPIG).
    \section{ESTIMATION}\label{sec:estimation}

\subsection{Computing EUR by mutual information for classification tasks}\label{app:clf_estimation}
We start by deriving the weighted version of EPIG:
\begin{align*}
    \mathrm{EPIG}_w&=\expectation{p(c)}{\mathrm{EUR}^{\phi,T, w}_{p(z,y|c,x)}(x;c)}\\
    &= \expectation{p(c)}{\uncertaintyphiw{p(z|c)} - \expectation{p(y|x,c)}{\uncertaintyphiw{p(z|c,x,y)}}}\\
    &=\expectation{p(c)}{\expectation{p(z|c)}{-w(z)\log p_w(z|c)} - \expectation{p(z,y|c,x)}{-w(z)\log p_w(z|c,y,x)}}\\
    &= \expectation{p(c)p(z,y|c,x)}{w(z)\log\frac{p_w(z|c,y,x)}{p_w(z|c)}}\\
    &=\mathbb{E}_{p(c)\,p(z\mid c)\,p(y\mid x,z,c)}
    \Bigl[w(z)\,\log\frac{p(y\mid x,z,c)}{q(y\mid x,c)}\Bigr],
\end{align*}
where the last equation is derived by \cref{eq:rewighted_bayes_rule}.
In terms of implementation, we use the formula for mutual information directly to reduce computational cost.

\begin{align*}
    \mathrm{EPIG}_w
    &=\mathbb{E}_{p(c)\,p(z\mid c)\,p(y\mid x,z,c)}
    \Bigl[w(z)\,\log\frac{p(y\mid x,z,c)}{q(y\mid x,c)}\Bigr]
    \\
    &=\mathbb{E}_{p(c)\,p(z\mid c)\,p(y\mid x,z,c)}
    \left[w(z)\,\log\frac{p(y\mid x,z,c)}
    {\mathbb{E}_{p_w(z\mid c)}\bigl[p(y\mid x,z,c)\bigr]}\right],
    \\
    &=\mathbb{E}_{p(c)\,p(z, y\mid c,x)}
    \left[
      w(z)\,\log\frac{p(y\mid x,z,c)}
      {\mathbb{E}_{p(z\mid c)}\bigl[\frac{w(z)}{\mathbb{E}_{p(z\mid c)}[w(z)]}\,p(y\mid x,z,c)\bigr]}
    \right]\\
    &=\mathbb{E}_{p(c)\,p(z, y\mid c,x)}
    \left[
      w(z)\,\log\frac{w(z)p(z, y\mid c, x)}
      {\frac{w(z)}{\mathbb{E}_{p(z\mid c)}[w(z)]}p(z\mid c)\mathbb{E}_{p(z\mid c)}[w(z)\,p(y\mid x,z,c)]}
    \right]\\
    &=\mathbb{E}_{p(c)}\left[\sum_{z, y}w(z)\,p(z, y\mid c, x)\,\log\frac{w(z)p(z, y\mid c, x)}
      {p_w(z\mid c)\sum_z w(z)\,p(z, y\mid c, x)}\right].
\end{align*}
Following the notation and strategy in \cite{bickfordsmith2023prediction}, suppose we have samples from $\theta_i\sim p(\theta)$ and $c^j\sim p(c)$. The predictive distribution needed can be estimated as
\begin{align*}
    \hat{p}_n\left(z, y\mid c^j, x\right) & =\frac{1}{K} \sum_{i=1}^{K} p\left(z \mid c^{j}, \theta_{i}\right)\,p\left(y \mid x, \theta_{i}\right) \\
    \hat{p}_n\left(z \mid c^{j}\right) & =\frac{1}{K} \sum_{i=1}^{K} p\left(z \mid c^{j}, \theta_{i}\right)\\
    \hat p_w(z\mid c^j) &= \frac{w(z)}{\sum_z \hat p(z\mid c^j)w(z)}\cdot\hat p(z\mid c^j).
\end{align*}

\subsection{Exact one-step posterior update for variance-based acquisition}\label{app:reg_estimation}
We consider the GP regression setting from \cref{sec:exp_reg}. Let $\mathcal{D}_n=\{(x_i,y_i)\}_{i=1}^n$ with inputs $X_n=[x_1,\dots,x_n]^\top$ and responses $\bm y=[y_1,\dots,y_n]^\top$. The prior model is defined as
\begin{align*}
    f \sim \mathcal{GP}(0,k),\qquad y_i=f(x_i)+\varepsilon_i,\ \ \varepsilon_i\stackrel{\text{i.i.d.}}{\sim}\mathcal N(0,\sigma^2),
\end{align*}
with fixed hyperparameters. Denote by
\begin{align*}
    A_n := K_{X_nX_n}+\sigma^2 I,\qquad
    \bm k(x,X_n):=\big[k(x,x_1),\ldots,k(x,x_n)\big]^\top ,
\end{align*}
and write the posterior predictive mean and (co)variance under $\mathcal{D}_n$ as
\begin{align*}
    m_n(x) &= \bm k(x,X_n)^\top A_n^{-1}\bm y,\\
    v_n(x,x') &= k(x,x')-\bm k(x,X_n)^\top A_n^{-1}\bm k(x',X_n),\qquad v_n(x):=v_n(x,x).
\end{align*}

Let $x^+$ be a candidate pooled input and $y^+$ the (as yet unobserved) label at $x^+$.
The one\textendash step posterior after hypothetically observing $(x^+,y^+)$ is Gaussian with
\citep[Sec.~2.4]{williams2006gaussian}
\begin{align}
    m_{n+1}(x) &= m_n(x) + \beta_n(x;x^+)\,\big(y^+-m_n(x^+)\big),\label{eq:gp-mean-update}\\
    v_{n+1}(x,x') &= v_n(x,x') - \beta_n(x;x^+)\,\beta_n(x';x^+)\,\tau_n^2(x^+),\label{eq:gp-cov-update}
\end{align}
where
\begin{align*}
    \beta_n(x;x^+) := \frac{\mathrm{cov}_n\!\big(f(x),y^+\big)}{\mathrm{var}_n(y^+)}
    =\frac{v_n(x,x^+)}{v_n(x^+)+\sigma^2},\qquad
    \tau_n^2(x^+) :=v_n(x^+)+\sigma^2 .
\end{align*}

\paragraph{Expected variance reduction (EVR)}
Let $c\sim p(c)$ denote the distribution over contexts (test inputs) for which we evaluate uncertainty
(\cref{sec:method}). Since $v_{n+1}(x,x)$ in~\eqref{eq:gp-cov-update} does not depend on the realised $y^+$,
$\mathbb E_{p(y^+\mid x^+)}[v_{n+1}(c,c)]=v_{n+1}(c,c)$.
Thus the squared\textendash loss EUR (EVR) for candidate $x^+$ can be approximated by samples of $c$:
\begin{align}
    \mathrm{EVR}(x^+) 
    \approx\frac{1}{M}\sum_{j=1}^M\!\Big(v_n(c_j)-\mathbb E_{p(y^+\mid x^+)}\!\big[v_{n+1}(c_j,c_j)\big]\Big)
    = \frac{1}{M}\sum_{j=1}^M\frac{v_n(c_j,x^+)^2}{v_n(x^+)+\sigma^2}.\label{eq:evr-closed-form}
\end{align}
This is exactly the expected reduction in the Jensen gap $h_\phi$ for $\phi(x)=x^2$, averaged over contexts, and corresponds to the unweighted EUR in \cref{sec:method}. This score is also equivalent (up to constants) to minimising $\mathrm{tr}\left(\mathrm{Var}[f_A\mid D,y_x]\right)$ in variance-based transductive active learning (VTL), with $A$ corresponding to the context set \citep{hubotter2024transductive}.

\paragraph{Weighted expected variance reduction ($\textbf{EVR}_w$)}

To encode region-specific preferences, fix $w:\mathbb R\to(0,\infty)$ and, for a context $c\sim p(c)$, define the weighted Bregman uncertainty (squared-error case) under the current predictive $z\sim\mathcal N\!\big(m_n(c),v_n(c)\big)$:
\begin{align*}\label{eq:wqu-def}
    U_w^{(n)}(c)
    := \mathbb E\!\left[w(z)\big(z-\mu_w^{(n)}(c)\big)^2\right],\qquad
    \mu_w^{(n)}(c):=\frac{\mathbb E\!\left[w(z)\,z\right]}{\mathbb E\!\left[w(z)\right]} .
\end{align*}
By the reweighting identity in \cref{sec:method}, $U_w^{(n)}(c)=\bar w_n(c)\cdot \mathrm{Var}_{p_w} (z)$ with $p_w(\cdot)\propto w(\cdot)\,\mathcal N(m_n(c),v_n(c))$ and $\bar w_n(c)=\mathbb E_{p}[w(z)]$; for general $w$ the moments are not closed form, so we estimate them by Monte Carlo.

After updating with $(x^+,y^+)$, $z\mid y^+\sim \mathcal N\big(m_{n+1}(c),\,v_{n+1}(c)\big)$ with mean/variance from \cref{eq:gp-mean-update,eq:gp-cov-update}. We approximate the expectation over the GP predictive law $p(y^+\mid x^+)=\mathcal N\big(m_n(x^+),\tau_n^2(x^+)\big)$ using common random numbers:
\begin{align*}
    y^{+,(j)} &= m_n(x^+)+\sqrt{\tau_n^2(x^+)}\,\eta^{(j)},\qquad \eta^{(j)}\stackrel{\text{i.i.d.}}{\sim}\mathcal N(0,1),\\
    z_{\cdot\mid +}^{(i,j)} &=m_n(c)+\beta_n(c_j;x^+)\sqrt{\tau_n^2(x^+)}\,\eta^{(j)}+\sqrt{v_{n+1}(c_j)}\,\varepsilon^{(i)},\qquad \varepsilon^{(i)}\stackrel{\text{i.i.d.}}{\sim}\mathcal N(0,1),
\end{align*}
reusing the same $\{\eta^{(j)}\}$ and $\{\varepsilon^{(i)}\}$ across candidates $x^+$ to reduce estimator variance.
By jointly sampling from $(y^+,z, c)$, the expected posterior weighted uncertainty is approximated by
\begin{align*}
    \mathbb E_{p(y^+\mid x^+)}\!\left[U_w^{(n+1)}(c_j)\right]\approx
    \frac{1}{S}\sum_{i=1}^{S}
    w\big(z_{\cdot\mid +}^{(i,j)}\big)\left(z_{\cdot\mid +}^{(i,j)}-\mu_{w}^{(j)}(c)\right)^2,
    \quad
    \mu_{w}^{(j)}(c):=\frac{\sum_{i=1}^{S} w\big(z_{\cdot\mid +}^{(i,j)}\big)\,z_{\cdot\mid +}^{(i,j)}}{\sum_{i=1}^{S} w\big(z_{\cdot\mid +}^{(i,j)}\big)}.
\end{align*}
The weighted EUR score averages the reduction across contexts:
\begin{align*}
    \mathrm{EVR}_w(x^+)\approx\frac{1}{M}\sum_{j=1}^M
    \Big(
    U_w^{(n)}(c_j)-\mathbb E_{p(y^+\mid x^+)}\!\big[U_w^{(n+1)}(c_j)\big]
    \Big).
\end{align*}
\emph{Remark.} For exponential weights $w(z)=\exp(\alpha z)$, the reweighting is analytic:
$p_w=\mathcal N\big(m_n(c)+\alpha v_n(c),\,v_n(c)\big)$,
$\bar w_n(c)=\exp\!\big(\alpha m_n(c)+\tfrac{1}{2}\alpha^2 v_n(c)\big)$,
and $U_w^{(n)}(c)=\bar w_n(c)\,v_n(c)$, yielding a closed\textendash form pre\textendash update term.

\paragraph{Implementation notes}

Compute and cache the Cholesky $A_n= L L^\top$.
For any candidate $x^+$, obtain
\begin{align*}
    s_n(\mathcal C,x^+):=v_n(\mathcal C,x^+)=k(\mathcal C,x^+)-K_{\mathcal C X_n}\,A_n^{-1}\bm k(X_n,x^+),
    \qquad
    \tau_n^2(x^+)=v_n(x^+)+\sigma^2,
\end{align*}
then evaluate \eqref{eq:evr-closed-form} via $\mathrm{EVR}(x^+)=\tfrac{1}{M}\sum_{j=1}^M s_n(c_j,x^+)^2/\tau_n^2(x^+)$.
The same cached solves provide $\beta_n(c;x^+)=s_n(c,x^+)/\tau_n^2(x^+)$ and $v_{n+1}(c)=v_n(c)-s_n(c,x^+)^2/\tau_n^2(x^+)$ for the weighted estimator above; only the nonlinear weighting $w(\cdot)$ requires Monte Carlo. Note that we compute $\text{EVR}_w$ to empirically demonstrate its non-negativity. To reduce computation, one could also minimise expected predictive posterior score over contexts.

\subsection{General techniques for variance-based uncertainty estimation}
\label{app:variance-estimation}

We briefly simplify the objectives used to score a candidate query $x$ and provide low-variance Monte Carlo estimators.

\paragraph{Setup.}
Let $c\sim p(c)$ denote a context (e.g.\ a test input), $z\in\mathbb R$ a scalar predictive target, and $p(\cdot)$ the predictive distribution after $n$ observations. We consider the weighted squared-loss geometry with a strictly positive weight function $w:\mathbb R\to (0,\infty)$ and the reweighted belief
\[
p_w(z\mid c)\;\propto\; w(z)\,p(z\mid c), \qquad 
\bar w(c)\;:=\;\mathbb E_{p(z\mid c)}[w(z)].
\]
In this geometry the uncertainty is the non-negative Jensen gap
\(
h_\varphi[p(z\mid c)]=\mathrm{Var}_{p(z\mid c)}[z]
\)
(and its weighted counterpart $\bar w(c)\,\mathrm{Var}_{p_w(z\mid c)}[z]$), so EUR reduces to (weighted) expected predictive-variance reduction; see Secs.~3.3–3.4 and \cref{app:epu_retrive}.

\paragraph{EVR.}
The unweighted expected variance reduction is
\begin{align*}
\mathrm{EVR}(x)
&=\mathbb E_{p(c)}\;\mathbb E_{p(y\mid x)}
\!\left[\mathrm{Var}\!\left(z\mid c\right)-\mathrm{Var}\!\left(z\mid c,y,x\right)\right] \\
&=\text{constant}+\mathbb E_{p(c)}\;\mathbb E_{p(y\mid x)}
\!\left[\big(\mathbb E_{p(z\mid c,y,x)}[z]\big)^2\right].
\end{align*}
Using two i.i.d.\ replicates $z,z'\stackrel{\text{i.i.d.}}{\sim}p(z\mid c,y,x)$ as per~\citet{rainforth2018nesting} yields $\big(\mathbb E[z]\big)^2=\mathbb E[zz']$ and thus
\begin{align*}
\mathrm{EVR}(x)&=\text{constant}+\mathbb E_{p(c)}\;\mathbb E_{p(y\mid x)}\;\mathbb E_{p(z\mid c,y,x)p(z'\mid c,y,x)}\!\left[zz'\right], \\
&=\text{constant}+\mathbb E_{p(c)}\;\mathbb E_{p(z|c)}\;\mathbb E_{p(y\mid x, z, c)p(z'\mid c,y,x)}\!\left[zz'\right], \\
&=\text{constant}+\mathbb E_{p(c)}\;\mathbb E_{p(\theta)p(z\mid\theta,c)p(y\mid x,\theta,c)}\;\mathbb E_{p(\theta'\mid y,x) p(z'\mid \theta', c)}\!\left[zz'\right], \\
&=\text{constant}+\mathbb E_{p(c)}\;\mathbb E_{p(\theta)p(y\mid x,\theta,c)p(z\mid \theta, c)}\;\mathbb E_{p(\theta')p(z'\mid \theta', c)}
\!\left[\frac{p(y\mid x, \theta', c)}{\mathbb{E}_{p(\theta'')}\left[p(y\mid x, \theta', c)\right]}  zz'\right], \\
&=\text{constant}+\mathbb E_{p(c)}\;\mathbb E_{p(\theta)p(y\mid x,\theta,c)}\;\mathbb E_{p(\theta')}
\!\left[\frac{p(y\mid x, \theta', c)}{\mathbb{E}_{p(\theta'')}\left[p(y\mid x, \theta'', c)\right]} \mathbb{E}_{p(z\mid \theta, c)} \left[z\right] \mathbb{E}_{p(z'\mid \theta', c)} \left[z'\right]\right].
\end{align*}
Either the final or penultimate line can now be estimated using nested Monte Carlo~\citep{rainforth2018nesting} or adapting a contrastive estimator like PCE~\citep{foster2020unified}, while if we have a mechanism for drawing (approximate) samples from $p(z'\mid c,y,x)$, then we can use the second line directly instead.
One can also replace the sampling from any of the prior terms, e.g.~$p(\theta''|c)$, with an appropriate importance sampler instead if desired.

\paragraph{Weighted EVR.}
With prediction-space weighting, the objective becomes
\begin{align*}
\mathrm{EVR}_w(x)
&=\mathbb E_{p(c)}\;\mathbb E_{p(y\mid x)}
\!\left[\bar w(c)\,\mathrm{Var}_{p_w(z\mid c)}[z]\;-\;\bar w(c,y,x)\,\mathrm{Var}_{p_w(z\mid c,y,x)}[z]\right]\\
&=\text{constant} - \mathbb E_{p(c)}\;\mathbb E_{p(y\mid x)}
\!\left[\bar w(c,y,x)\,\mathrm{Var}_{p_w(z\mid c,y,x)}[z]\right].
\end{align*}
Using $\bar w\,\mathrm{Var}_q[z]
=\mathbb E_{p}[w(z)z^2]-\frac{\big(\mathbb E_{p}[w(z)z]\big)^2}{\mathbb E_{p}[w(z)]}$ (where expectations are under the indicated $p(z\mid\cdot)$), 
and since $\mathbb E_{p(y\mid x)}\!\big[\mathbb E_{p(z\mid c,y,x)}[w(z)z^2]\big]$ is constant in $x$ by marginalisation, we then have
\begin{align*}
 &\mathrm{EVR}_w(x)-\text{constant}=
\;+\;\mathbb E_{p(c)}\;\mathbb E_{p(y\mid x)}
\!\left[\frac{\big(\mathbb E_{p(z\mid c,y,x)}[w(z)z]\big)^2}{\mathbb E_{p(z\mid c,y,x)}[w(z)]}\right] \\
&= \mathbb E_{p(c)}\;\mathbb E_{p(y\mid x)p(z\mid c,y,x)p(z'\mid c,y,x)}
\!\left[\frac{zz'w(z)w(z')}{\mathbb E_{p(\theta''|y,x,c)p(z\mid \theta'', c)}[w(z)]}\right] \\
&= \mathbb E_{p(c)}\;\mathbb E_{p(\theta)p(y\mid x,\theta,c)p(z\mid \theta,c)} \mathbb{E}_{p(\theta'\mid c,y,x)p(z'|\theta',c)}
\!\left[\frac{zz'w(z)w(z') \mathbb{E}_{p(\theta'')}\left[p(y|x,\theta'',c)\right]}{\mathbb E_{p(\theta'')p(z\mid \theta'', c)}[w(z)p(y|x,\theta'',c)]}\right] \\
&= \mathbb E_{p(c)}\;\mathbb E_{p(\theta)p(y\mid x,\theta,c)p(z\mid \theta,c)} \mathbb{E}_{p(\theta')p(z'|\theta',c)}
\!\left[\frac{p(y|x,\theta',c)}{\mathbb E_{p(\theta'')p(z\mid \theta'', c)}[w(z)p(y|x,\theta'',c)]}zz'w(z)w(z')\right]\\
&= \mathbb E_{p(c)}\;\mathbb E_{p(\theta)p(y\mid x,\theta,c)} \mathbb{E}_{p(\theta')}
\!\left[\frac{p(y|x,\theta',c)}{\mathbb E_{p(\theta'')p(z\mid \theta'', c)}[w(z)p(y|x,\theta'',c)]} \mathbb{E}_{p(z\mid \theta,c)}[zw(z)] \mathbb{E}_{p(z'\mid \theta',c)}[z'w(z')] \right].
\end{align*}
We can now again do a nested Monte Carlo or PCE-style estimator for the nested estimations given in the last and penultimate lines in the same way as for the unweighted case with the same underlying computational complexity.

\subsection{Method complexity}

\paragraph{Classification.}
The weighted EPIG estimator, $\text{EPIG}_w$ (see \cref{app:clf_estimation}), has per–candidate-input complexity $\mathcal{O}(MK)$, where $M$ is the number of Monte Carlo draws over contexts and joint predictive outputs, and $K$ is the number of samples of model parameters. This matches the asymptotic cost of standard (unweighted) EPIG estimators \cite{bickfordsmith2023prediction}.

\paragraph{Regression.}
For regression we use expected variance reduction (EVR) and its weighted counterpart ($\text{EVR}_w$). In general, both have complexity $\mathcal{O}(MS)$, with $S$ inner Monte Carlo draws used to estimate weighted posterior means and/or variances. In models that admit an exact one-step posterior update—e.g., exact Gaussian processes where the predictive variance after a single observation is available in closed form—the per–candidate complexity reduces to $\mathcal{O}(M)$ because no parameter sampling is required.

    \section{EXPERIMENT DETAILS}\label{sec:experiment_details}

\subsection{Metrics}\label{app:metrics}
Our primary metrics match the losses assumed in the acquisition objective. 
For classification we report negative log likelihood (NLL) and its weighted counterpart,
\begin{equation*}
    \mathrm{NLL}_w=\frac{\mathbb{E}_{p_{\mathrm{eval}}(y,x)}\!\big[w(y)\, \big(-\log q_{n+m}(y\mid x)\big)\big]}
     {\mathbb{E}_{p_{\mathrm{eval}}(y)}[w(y)]}
     .
\end{equation*}
For regression we report squared-error loss (SEL) and its weighted counterpart,
\begin{equation*}
    \mathrm{SEL}_w=\frac{\mathbb{E}_{p_{\mathrm{eval}}(y,x)}\!\big[w(y)\, \left(y-\mathbb{E}_{q_{n+m(y|x)}}[y]\right)^2\big]}
     {\mathbb{E}_{p_{\mathrm{eval}}(y)}[w(y)]}.
\end{equation*}
Note that we do not record accuracy as an evaluation metric, as accuracy implies a zero-one loss that we do not optimise for during data acquisition. Therefore, there is no reason for acquisition methods to dominate on accuracy-related metrics.

In addition, we do not claim that each benchmark has a uniquely correct weighting. Rather, the experiments are designed to test whether the acquisition rule can be customised to different user-specified downstream preferences. Accordingly, all weights are fixed a priori from simple task-motivated heuristics and are not tuned on validation performance.

\subsection{GP model setting}\label{app:gp_model_setting}

We use a Gaussian process (GP) regressor with a constant prior mean $m(x)\equiv c$ set to the median of the observed responses; we train on $y-c$ and add $c$ back at prediction. Inputs are standardized feature‑wise once on the initial train+pool set and the same transform is reused thereafter; the target is left in native units. The kernel is a sum of a linear (dot‑product) term and a Matérn term with i.i.d.\ white noise, i.e.,
$m(x)\equiv c_t$ and
$k(x,x')=\sigma_{\mathrm{lin},t}^2\,z^\top z' + \sigma_{f,t}^2\,\kappa_{\nu}\!\big(\|z-z'\|/\ell_t\big) + \sigma_{n,t}^2\,\delta_{x,x'}$,
where $z$ denotes standardized inputs (we standardize $X$ only), $\kappa_\nu$ is the Matérn correlation (we use $\nu=3/2$ unless noted), and $\delta_{x,x'}$ is the Kronecker delta \citep{williams2006gaussian,Stein1999}.
This is the usual universal‑kriging decomposition of a linear trend plus a stationary residual \citep{Cressie1993,DiggleRibeiro2007}.  Hyperparameters $(c_t,\sigma_{\text{lin},t}^2,\sigma_{f,t}^2,\sigma_{n,t}^2,\ell_t)$ are recomputed by robust plug‑in rules every three acquisitions and held fixed in between; no marginal-likelihood maximisation or other gradient-based tuning is used.
We fix $\nu{=}5/2$ on \textsc{Yacht} and $\nu{=}3/2$ on \textsc{Slump}/\textsc{Estate}; $\nu{=}5/2$ induces smoother (twice m.s.\ differentiable) sample paths than $\nu{=}3/2$ (once m.s.\ differentiable), which we found to better match the hydrodynamics response, whereas the latter is more robust for noisier, less‑smooth tabular targets.

\subsection{Robust hyperparameter estimation}\label{app:hyperparam_estimation}

Following the GP setting above, we set the location $c_t=\mathrm{median}(y)$ (robust location) and estimate scales by robust/difference‑based rules: (i) the observation noise is estimated from nearest‑neighbor differences in standardized input space,
$\sigma_{n,t}=\tfrac{1.4826}{\sqrt{2}}\ \mathrm{median}_i\,|y_i-y_{j(i)}|$, $j(i)=\arg\min_{j\neq i}\|z_i-z_j\|$, a difference‑based variance estimator in the spirit of \citet{Rice1984} and \citet{HallKayTitterington1990}, with the MAD‑to‑$\sigma$ factor $1.4826$ ensuring normal Fisher‑consistency \citep{HuberRonchetti2009,Tyler2011}; (ii) we split linear vs.\ residual signal by a small‑ridge fit on $Z$ (standardized inputs) to $y-c_t$, taking $\sigma_{\mathrm{lin},t}^2=\mathrm{Var}(\hat y)$ and $\sigma_{f,t}^2=\max\!\big((1.4826\,\mathrm{median}|r|)^2-\sigma_{n,t}^2,\varepsilon\big)$ with residuals $r=(y-c_t)-\hat y$ and $\varepsilon\!=\!10^{-12}$ \citep{HoerlKennard1970,HuberRonchetti2009}; (iii) we set the length‑scale $\ell_t$ by method‑of‑moments matching: choose a characteristic design spacing $r_{0,t}$ (the median nearest‑neighbor distance in $z$) and solve $\kappa_\nu(r_{0,t}/\ell_t)=\rho$ with a target short‑lag correlation $\rho=0.5$ (a “practical range’’ style calibration; closed form for $\nu=3/2$ gives $\ell_t\approx1.032\,r_{0,t}$) \citep{Cressie1993,Stein1999,DiggleRibeiro2007}. These choices yield a prior that reflects the data scale and geometry while remaining fully predetermined between update rounds.
    \section{EXTRA RESULTS}

\subsection{Regression}\label{app:reg_fig}

With the same model setup in \cref{sec:exp_reg_1d}, \cref{fig:reg_inv} shows the data-gathering procedure when $w(z)=\exp(-z)$, which targets the precision in the low-value region. After a warm-start of exploring potential low value, the target region is secured and more data is acquired around the area.

\subsection{Classification}\label{app:clf_evo_fig}

We present \cref{fig:clf_proportion} to demonstrate the evolution for proportion of acquired class for Vehicle, Landsat and Vowel, respectively. The weighting has a one-to-one correspondence with the class in proportion plot. Curves show cumulative proportion of acquired labels per class for each method (Random, EPIG, $\text{EPIG}_w$) by the mean over 100 runs; shaded regions indicate the SEM range. The class-mix plots confirm that $\text{EPIG}_w$ allocates a larger share of queries to the high-weight classes compared to EPIG and Random.
    \section{RESOURCES}

\subsection{Software}

\begin{table}[h]
    \centering
    \scriptsize
    \begin{tabular}{llll}
        \toprule
        Project        & Citation                         & License            & URL                                             \\
        \midrule
        NumPy          & \citet{harris2020array}          & BSD (3-clause)     & \shorturl{numpy.org}                            \\
        SciPy          & \citet{virtanen2020scipy}        & BSD (3-clause)     & \shorturl{scipy.org}                            \\
        Scikit-learn   & \citet{pedregosa2011sklearn}     & BSD (3-clause)     & \shorturl{scikit-learn.org}                     \\
        Matplotlib     & \citet{hunter2007matplotlib}     & PSF-based          & \shorturl{matplotlib.org}                       \\
        pandas         & \citet{mckinney2010data}         & BSD (3-clause)     & \shorturl{pandas.pydata.org}                    \\
        tqdm           & \citet{dacostaluis2019tqdm}      & MIT                & \shorturl{github.com/tqdm/tqdm}                 \\
        PyTorch        & \citet{paszke2019pytorch}        & BSD (3-clause)     & \shorturl{pytorch.org}                          \\
        GPyTorch       & \citet{gardner2018gpytorch}      & MIT                & \shorturl{gpytorch.ai}                          \\
        h5py           & \citet{h5pydocs2025}             & BSD (3-clause)     & \shorturl{h5py.org}                             \\
        PMLB           & \citet{olson2017pmlb,romano2021pmlb} & MIT            & \shorturl{epistasislab.github.io/pmlb}          \\
        ucimlrepo (UCI client) & \citet{uci2025citation}  & MIT                & \shorturl{github.com/uci-ml-repo/ucimlrepo}     \\
        \bottomrule
    \end{tabular}
    \caption{Third‑party software used in this work.}
    \label{tab:software}
\end{table}

\subsection{Datasets}

\begin{table}[h]
    \centering
    \scriptsize
    \begin{tabular}{lll}
        \toprule
        Dataset & Citation & URL \\
        \midrule
        \textsc{Vehicle} & \cite{statlog_vehicle} & \shorturl{archive.ics.uci.edu/dataset/149/statlog+vehicle+silhouettes} \\
        \textsc{Landsat} & \cite{statlog_landsat} & \shorturl{archive.ics.uci.edu/dataset/146/statlog+landsat+satellite} \\
        \textsc{Vowel} & \cite{connectionist_bench_vowel} & \shorturl{archive.ics.uci.edu/dataset/152/connectionist+bench+vowel+recognition+deterding+data} \\
        \textsc{Slump} & \cite{concrete_slump_test_182} & \shorturl{archive.ics.uci.edu/dataset/182/concrete+slump+test} \\
        \textsc{Yacht} & \cite{yacht_hydrodynamics_243} & \shorturl{archive.ics.uci.edu/dataset/243/yacht+hydrodynamics} \\
        \textsc{Estate} & \cite{real_estate_valuation_477} & \shorturl{archive.ics.uci.edu/dataset/477/real+estate+valuation+data+set} \\
        \bottomrule
    \end{tabular}
    \caption{UCI datasets used in this work. All are available under Creative Commons Attribution 4.0 (CC BY~4.0).}
    \label{tab:uci_datasets}
\end{table}

\subsection{Compute}
All experiments were executed on a single local workstation; no distributed or cloud resources were used. The software environment is listed in Table~\ref{tab:software}, and the datasets are summarized in Table~\ref{tab:uci_datasets}.

\begin{figure}[h]
    \centering
    \includegraphics[width=\linewidth]{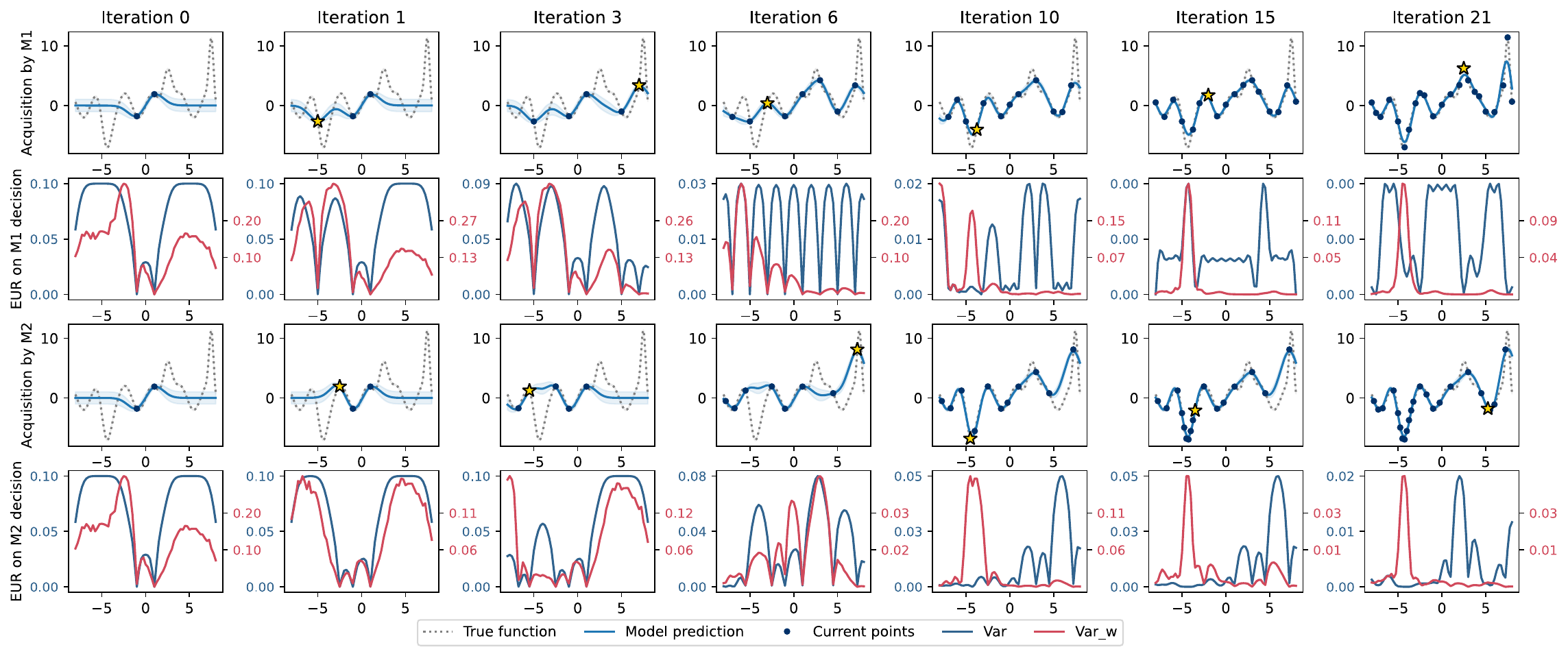}
    \caption{
        Expected variance reduction (\blue{\text{EVR}}) vs.\ a variant with weighting $w(z)=\exp(-z)$ (\red{$\text{EVR}_w$}).
        Row~1: prediction with data acquired by \blue{\text{EVR}}.
        Row~2: values of \blue{\text{EVR}} and \red{$\text{EVR}_w$} given the Row~1 training set (the maximiser is labelled next).
        Row~3: prediction with data acquired by \red{$\text{EVR}_w$}.
        Row~4: values of \blue{\text{EVR}} and \red{$\text{EVR}_w$} given the Row~3 training set.
    }
    \label{fig:reg_inv}
    
\end{figure}

\begin{figure}[h]
    \centering
    \includegraphics[width=0.7\linewidth]{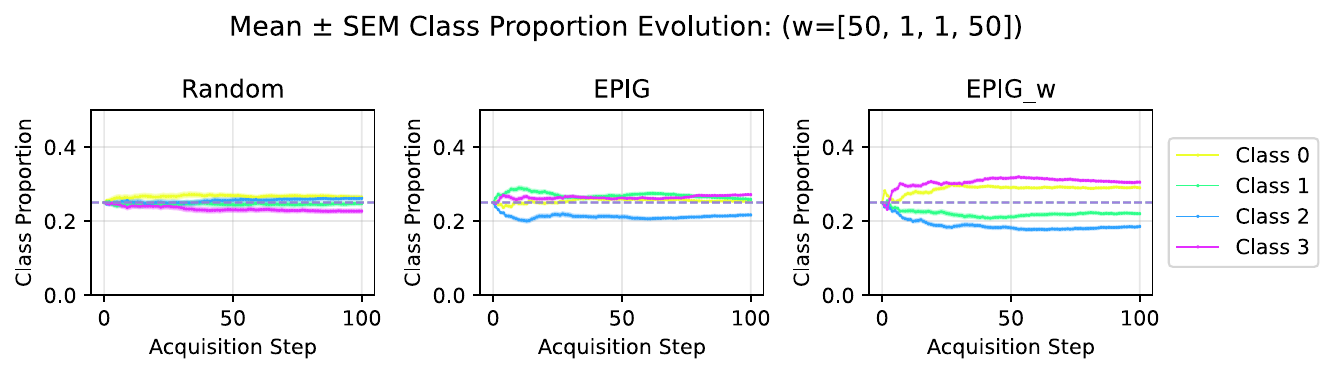}
    \includegraphics[width=0.7\linewidth]{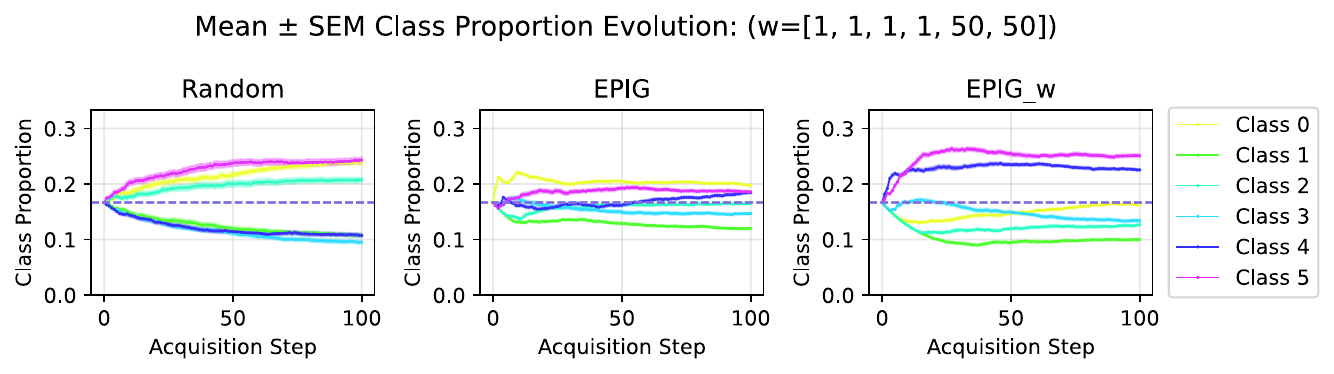}
    \includegraphics[width=0.7\linewidth]{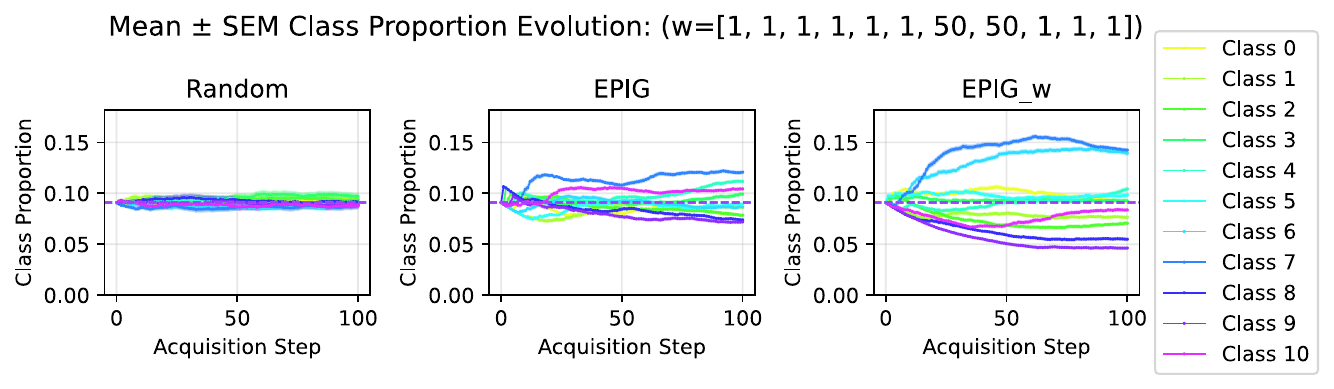}
    \caption{
        From top to bottom: evolution of acquired-class proportions on \textsc{Vehicle} ($w=[50,1,1,50]$), \textsc{Landsat} ($w=[1,1,1,1,50,50]$) and \textsc{Vowel} ($w=[1,1,1,1,1,1,50,50,1,1,1]$).
    }
    \label{fig:clf_proportion}
    \vspace{-70pt}
\end{figure}
\end{document}